\documentclass{article}

\PassOptionsToPackage{numbers, compress}{natbib}

\usepackage[preprint]{neurips_2026}

\usepackage[utf8]{inputenc} 
\usepackage[T1]{fontenc}    
\usepackage[colorlinks=true,
            linkcolor=red,
            citecolor=green,
            urlcolor=cyan]{hyperref}
\usepackage{url}            
\usepackage{booktabs}       
\usepackage{amsfonts}       
\usepackage{nicefrac}       
\usepackage{microtype}      
\usepackage{xcolor}         

\usepackage{bm}
\usepackage{amsmath,amssymb,amsthm}
\usepackage{mathtools}
\usepackage[capitalize,noabbrev]{cleveref}
\usepackage{autonum}
\usepackage{multirow}
\usepackage{booktabs}

\usepackage{algorithm}
\usepackage{algpseudocode}
\AtBeginEnvironment{algorithmic}{\ttfamily}
\usepackage{inconsolata}

\newcommand{\argmin}{\mathop{\rm argmin}\limits}

\DeclareMathOperator{\tr}{tr}

\DeclareMathOperator{\EE}{\mathbb{E}}

\DeclarePairedDelimiterX\inner[2]{\langle}{\rangle}{{#1},{#2}}

\DeclarePairedDelimiter\bra{[}{]}
\DeclarePairedDelimiterX\Set[2]{\{}{\}}{\mspace{2mu}{#1}\;\delimsize|\;{#2}\mspace{2mu}}
\DeclarePairedDelimiterX\Prn[2]{(}{)}{\mspace{2mu}{#1}\;\delimsize|\;{#2}\mspace{2mu}}
\DeclarePairedDelimiterX\Bra[2]{[}{]}{\mspace{2mu}{#1}\;\delimsize|\;{#2}\mspace{2mu}}

\newcommand{\R}{\mathbb R}

\renewcommand{\epsilon}{\varepsilon}

\NewDocumentCommand{\exsub}{s m O{} m}{%
  \IfBooleanT{#1}{\EE_{#2}\nolimits\bra*{#4}}%
  \IfBooleanF{#1}{\EE_{#2}\nolimits\bra[#3]{#4}}%
}

\usepackage{CJKutf8}

\usepackage{thmtools,thm-restate}

\declaretheorem[
  name=Proposition,
  refname={Proposition,Propositions},
]{proposition}
\declaretheorem[
  name=Lemma,
  refname={Lemma,Lemmas},
]{lemma}

\crefname{algorithm}{Algorithm}{Algorithms}
\crefname{line}{Line}{Lines}
\crefname{section}{Sec.}{Secs.}
\crefname{subsection}{Sec.}{Secs.}
\crefname{appendix}{Appendix}{Appendices}
\crefname{table}{Tab.}{Tabs.}
\crefname{figure}{Fig.}{Figs.}
\crefname{equation}{}{}

\Crefname{table}{Table}{Tables}
\Crefname{figure}{Figure}{Figures}
\Crefname{equation}{Eq.}{Eqs.}
\Crefname{section}{Section}{Sections}
\Crefname{subsection}{Section}{Sections}

\usepackage{mathrsfs}

\title{A Riemannian Geometry for Low-rank Adaptation}

\author{%
  Shoichiro Takeda \\
  NTT, Inc.\\
  \texttt{shoichiro.takeda@ntt.com}\\
  \And
  Shin'ya Yamaguchi \\
  NTT, Inc.\\
  \texttt{shinya.yamaguchi@ntt.com}\\
  \And
  Satoshi Suzuki\\
  NTT, Inc.\\
  \texttt{satoshixv.suzuki@ntt.com}\\
  \And
  Yasunori Akagi\\
  NTT, Inc.\\
  \texttt{yasunori.akagi@ntt.com}\\
}

\begin{document}

\maketitle



\begin{abstract}
Low-rank adaptation (LoRA) is widely used as a parameter-efficient fine-tuning technique for pre-trained deep neural networks, which approximates the weight update via full fine-tuning by a low-rank matrix $BA^\top$. This parameterization leads to the equivalence relation $(B, A) \sim (BG^{-1}, AG^\top)$ for any invertible matrix $G$ because $BA^\top = BG^{-1}(AG^\top)^\top$ and thus both pairs yield the same loss value. This relation induces a quotient manifold where matrices $(BG^{-1}, AG^\top)$ for all $G$ are identified, eliminating redundant directions along which the loss value remains unchanged. To respect the geometry of this manifold, the original search space is endowed with a Riemannian metric that is invariant under the equivalence relation. Such a metric induces preconditioning at each gradient step and ensures that each weight update via LoRA changes the loss value, leading to efficient optimization. In this paper, we propose a new Riemannian metric that is specifically tailored to LoRA to close the gap to full fine-tuning at the weight level. We theoretically show that LoRA with our preconditioning induced by this metric satisfies the following two properties at each iteration: (i) The weight update follows the direction closest to the gradient of full fine-tuning within the subspace of first-order weight changes allowed by the LoRA parameterization. (ii) The updated weight matrix is closer in Frobenius norm to that of full fine-tuning than the updated weight matrices of LoRA with conventional preconditioning and without preconditioning. These theoretical insights suggest that our preconditioning makes LoRA better approximate full fine-tuning, thereby leading to more efficient optimization. Experiments show the effectiveness and efficiency of our preconditioning for LoRA on fine-tuning tasks with language and vision domains.
\end{abstract}

\section{Introduction}\label{sec:introduction}
Recent progress of foundation models with deep neural networks, such as large vision and large language models, has been remarkable \cite{openai2024gpt4technicalreport,grattafiori2024llama3herdmodels,deepseekai2025deepseekv3technicalreport,yang2025qwen3technicalreport,bai2025qwen3vltechnicalreport}.
When transferring these pre-trained models to a specific downstream task, one popular technique is full fine-tuning that updates all the model parameters.
However, full fine-tuning of large foundation models is impractical due to their large number of parameters, requiring high computational and storage costs for each task-specific model.

To address this issue, parameter-efficient fine-tuning techniques have become a promising alternative to full fine-tuning \cite{li2021prefix,pfeiffer2021adapterfusion,hu2022lowrank,hu2023llmadapters,han2024parameterefficient}.
Among these techniques, low-rank adaptation (LoRA) is the de facto standard thanks to its simplicity and effectiveness \cite{hu2022lowrank}.
Let $W_0, W \in \R^{m \times n}$ denote the weight matrices of a layer in the pre-trained and fine-tuned models, respectively.
In LoRA, the weight update of full fine-tuning $\Delta W \in \R^{m \times n}$ from $W_0$ to $W$ is approximated by low-rank matrix factorization\footnote{We here omit the scaling $BA^\top$ by $s \in \R$ for simplicity as it can be merged into $B$ and $A$~\cite{li2025flatLoRA}.} as follows:
\begin{align}\label{eq:modeling_lora}
    W = W_0 + \Delta W \approx W_0 + BA^\top,
\end{align}
where $B \in \R^{m \times r}$ and $A \in \R^{n \times r}$ are trainable factors.
Let $\mathcal{L} : \R^{m \times n} \rightarrow \R$ denote the loss function for a downstream task.
Then, LoRA is formulated as
\begin{equation}\label{eq:lora}
    \min_{(B,A) \in \mathcal{X}}
    \mathcal{L}\left( W_0 + BA^\top \right),
\end{equation}
where $\mathcal{X} \coloneqq \R^{m \times r} \times \R^{n \times r}$ denotes the original search space for this problem.
By choosing a sufficiently small $r \ll \min(m,n)$, LoRA yields fewer trainable parameters and lower computational and storage costs than full fine-tuning, i.e., $\min_{W \in \R^{m \times n}} \mathcal{L}\left( W \right)$.

Due to the fewer trainable parameters, LoRA often exhibits a performance gap compared to full fine-tuning. 
To fill this gap, various approaches have been studied for improving the performance of LoRA, e.g., modeling (architecture) improvements~\cite{wu2024mosLoRA,li2025stiefel,park2025stiefel,dong2025aurora}, leveraging multiple LoRA modules~\cite{dou2024loramoe,luo2024moelora_constrative,chenghao2025greatagain,sun2025strongmixture,xia2024chainofLoRA,lin2025teamlora,liu2025seedlora}, learning-rate tuning~\cite{hayou2024lora}, and refined initialization schemes~\cite{meng2024pissa,wang2025milora,chenghao2025greatagain,wang2024loraga,zhang2025loraone}.
As another line of approach, exploiting the invariance of LoRA has been studied~\cite{zhang2024riemann,sun2025strongmixture}.
Specifically, the loss function $\mathcal{L}$ is invariant under the equivalence relation $(B, A) \sim (BG^{-1}, AG^\top)$, in the sense that $\mathcal{L}\,( W_0 + BA^\top) = \mathcal{L}\,( W_0 + BG^{-1}(AG^\top)^\top)$ for all $G \in \mathrm{GL}(r)$.
This equivalence relation induces the quotient manifold $\mathcal{X}/\mathrm{GL}(r)$ where matrices $(BG^{-1}, AG^\top)$ for all $G$ are identified, eliminating redundant directions along which the value of $\mathcal{L}$ remains unchanged.
To respect the geometry of this manifold, the original search space $\mathcal{X}$ is endowed with a Riemannian metric that is invariant under the equivalence relation.
Such a metric induces preconditioning at each gradient step and ensures that each weight update via LoRA changes the value of $\mathcal{L}$, leading to efficient optimization.
Conventional studies~\cite{zhang2024riemann,sun2025strongmixture} adopt a well-known invariant metric proposed in \cite{mishra2012riemanniangeometry}, and this metric induces preconditioning
\begin{equation}\label{eq:preconditioning_quotient}
    \nabla_{B}\mathcal{L} \,\,\mapsto\, (\nabla_{B}\mathcal{L}) (A^\top A)^{-1}, \quad
    \nabla_{A}\mathcal{L} \,\,\mapsto\, (\nabla_{A}\mathcal{L}) (B^\top B)^{-1},
\end{equation}
where $\nabla_{B}\mathcal{L}$ and $\nabla_{A}\mathcal{L}$ denote the Euclidean gradients of $\mathcal{L}$ with respect to $B$ and $A$, respectively.
LoRA with this preconditioning outperforms the original LoRA in multiple fine-tuning tasks.
However, the adopted Riemannian metric is classical and tuned to the low-rank matrix completion setting~\cite{mishra2012riemanniangeometry}.
Therefore, there remains room to tailor this metric to LoRA.

In this paper, we propose a new Riemannian metric to improve the performance of LoRA.
Specifically, we tailor the metric to LoRA to close the gap to full fine-tuning at the weight level.
As a result, this metric induces new preconditioning
\begin{equation}\label{eq:preconditioning_our}
    \nabla_{B} \mathcal{L} \,\,\mapsto\, ( I - \tfrac{1}{2} P_{B} ) (\nabla_{B} \mathcal{L}) (A^\top A)^{-1}, \quad
    \nabla_{A} \mathcal{L} \,\,\mapsto\, ( I - \tfrac{1}{2} P_{A} ) (\nabla_{A} \mathcal{L}) (B^\top B)^{-1},
\end{equation}
where $P_{B} \coloneqq B (B^\top B)^{-1} B^\top$ and $P_{A} \coloneqq A (A^\top A)^{-1} A^\top$ are the orthogonal projections onto the column spaces of $B$ and $A$, respectively.
We theoretically show that LoRA with this preconditioning satisfies the following two properties at each iteration $t$:
(i) The weight update from $W_{t}$ to $W_{t+1}$ follows the direction closest to the gradient of full fine-tuning within the subspace of first-order weight changes (i.e., first-order changes in $W$) allowed by the LoRA parameterization \cref{eq:modeling_lora}.
(ii) The updated weight matrix $W_{t+1}$ is closer in Frobenius norm to that of full fine-tuning than the updated weight matrices of LoRA with \cref{eq:preconditioning_quotient} and without preconditioning.
These theoretical insights suggest that our preconditioning makes LoRA better approximate full fine-tuning, thereby improving the performance of LoRA.
In addition, our preconditioning \cref{eq:preconditioning_our} admits an efficient implementation with the same time complexity of \cref{eq:preconditioning_quotient}.
We validated the practical effectiveness and efficiency of our preconditioning for LoRA on several fine-tuning tasks with language and vision domains, including natural language generation, image classification, and commonsense reasoning tasks.

\section{Related Works}
LoRA was initially formulated by \citet{hu2022lowrank} and achieves comparable full fine-tuning performance despite having very few trainable parameters.
However, it often exhibits a performance gap compared to full fine-tuning.
We here summarize several studies to improve the performance of LoRA.

{\bf Modeling (Architecture).}
Beyond the standard two-factor factorization \cite{hu2022lowrank,park2025stiefel,mishra2012riemanniangeometry,li2025flatLoRA} as described in \cref{eq:modeling_lora}, three-factor factorization has been proposed recently \cite{wu2024mosLoRA,li2025stiefel,park2025stiefel}, which inserts a small trainable matrix between $B$ and $A$ to fuse more information and flexibly.
Other approaches have also been explored, such as adapting low-rank values at each module \cite{zhang2023adaptive}, incorporating the mixture-of-experts architecture \cite{dou2024loramoe,luo2024moelora_constrative,chenghao2025greatagain,sun2025strongmixture}, exploiting multiple LoRA modules \cite{xia2024chainofLoRA,lin2025teamlora,liu2025seedlora}, and introducing non-linear mapping \cite{dong2025aurora}.
In this paper, we adopt the standard two-factor modeling \cref{eq:modeling_lora} as a first step and leave the expansion for other advanced modelings to future work.

{\bf Initialization.}
\citet{hu2022lowrank} originally proposed to use a random Gaussian initialization for $A$ and set zero for $B$ to ensure that $BA^\top$ is zero at the beginning of training.
Later, several initialization strategies have proposed to utilize the pre-trained weight characteristics \cite{meng2024pissa, wang2025milora}.
These methods differ in how they initialize $B$ and $A$, using the largest singular values and corresponding vectors of the pre-trained weight matrix \cite{meng2024pissa} or smaller ones \cite{wang2025milora}.
On the other hand, \citet{wang2024loraga} and \citet{zhang2025loraone} initialized $B$ and $A$ using the singular values and corresponding vectors of the one-step gradient of full fine-tuning.
In this paper, we follow the standard initialization proposed in \cite{hu2022lowrank} and the combination with the advanced initialization strategies is left for future work.

{\bf Optimization.}
Several optimization approaches for LoRA have been proposed \cite{wang2025lorapro, yue2025regsubspaecs, li2025stiefel, park2025stiefel}.
For example, \citet{wang2025lorapro} aligned the gradient of LoRA with that of full fine-tuning by solving a minimization problem for this gap\footnote{We further discuss the relationship between this method and our method in \cref{ap_sec:LoRAPro}.}.
To fully exploit low-dimensional subspaces spanned by the columns of $B$ and $A$, \citet{yue2025regsubspaecs} regularized those subspaces to capture complementary information.
Along this, \citet{li2025stiefel} and \citet{park2025stiefel} exploited the geometry of Stiefel manifold that imposes the orthogonal constraints on $B$ and/or $A$.
On the other hand, exploiting the invariance of LoRA has been studied \cite{zhang2024riemann,sun2025strongmixture}.
As explained in \cref{sec:introduction}, this approach induces preconditioning at each gradient step described in \cref{eq:preconditioning_quotient} and achieves efficient optimization compared to the original LoRA. 
This approach is the first Riemannian preconditioning study for LoRA, and its underlying idea is inspired by the classical low-rank matrix completion problem \cite{mishra2012riemanniangeometry}.
Note that, in the literature on low-rank matrix (or tensor) completion problem, several preconditioning variants have been proposed \cite{tong2022precontensor,zhang2024precon1,zhang2023precon2}.
Our work follows the Riemannian preconditioning for LoRA proposed in \cite{zhang2024riemann,sun2025strongmixture} but designs a Riemannian metric tailored to LoRA for more efficient optimization.

{\bf Others.}
To improve the performance of LoRA, several LoRA variants have been proposed.
\citet{hayou2024lora} applied different learning rates to $B$ and $A$ and shows fast convergence.
\citet{liu2024dora} decomposed the pre-trained weight into magnitude and direction, and then employed LoRA for directional updates to reduce the number of trainable parameters.
\citet{xia2024chainofLoRA} trained multiple LoRA sequentially like the residual learning to improve the performance of LoRA.
\citet{qlora} applied quantization to the pre-trained models with LoRA to save memory usage without sacrificing performance.
We leave the combination of these studies with our method to future work.

\section{Preliminary}
\subsection{Notations}
The set of real numbers is denoted by $\R$.
We denote the identity matrix and the zero matrix by $I$ and $O$, respectively, that omit the dimensional subscripts for simplicity.
The Frobenius norm for $X \in \R^{m \times n}$ is denoted by $\left\| X \right\|_{F} \coloneqq \sqrt{\tr(X^\top X)} = \sqrt{\sum_{i=0}^{m-1} \sum_{j=0}^{n-1} X_{ij}^2}$.
We denote the general linear group with size $r \times r$ by $\mathrm{GL}(r) = \{ G \in \R^{r \times r} \mid \operatorname{det}(G) \neq 0 \}$, where $\operatorname{det}(G)$ is the determinant of $G$.

\subsection{Riemannian Optimization}\label{sec:riemannian}
We review Riemannian optimization following~\cite{AbsMahSep2008,meyer2011factorization,bendokat2020Grassmannhandbook}, which will be used throughout this paper.

Let $\mathcal{M}$ be a Riemannian manifold endowed with a Riemannian metric $g_X : T_{X}\mathcal{M} \times T_{X}\mathcal{M} \rightarrow \R$ for each $X \in \mathcal{M}$, where $T_{X}\mathcal{M}$ denotes the tangent space of $\mathcal{M}$ at $X$.
This metric defines an inner product between two tangent vectors in $T_{X}\mathcal{M}$.
We consider an optimization problem
\begin{equation}\label{eq:obj_func_riemannian}
    \min_{X \in \mathcal{M}} f(X),
\end{equation}
where $f : \mathcal{M} \rightarrow \R$ is a smooth function. 
The Riemannian gradient of $f$ at $X \in \mathcal{M}$, denoted as $\operatorname{grad}_{X} f$, is defined as the unique tangent vector in $T_{X}\mathcal{M}$ satisfying
\begin{equation}\label{eq:def_rieman_grad}
    \mathrm{D}f(X)[\xi_{X}] = g_X \big(\operatorname{grad}_{X} f,\, \xi_{X} \big),\quad \forall \xi_{X} \in T_{X}\mathcal{M},
\end{equation}
where $\mathrm{D}f(X)[\xi_{X}]$ denotes the directional derivative of $f$ at $X$ along $\xi_{X}$.
Then, the $t$-th update of the Riemannian gradient descent method for \cref{eq:obj_func_riemannian} is given by
\begin{equation}\label{eq:update_riemannian}
    X_{t+1} = \mathcal{R}_{X_t} \big( - \alpha \operatorname{grad}_{X_t} f \big),
\end{equation}
where $\mathcal{R}_{X} : T_{X}\mathcal{M} \rightarrow \mathcal{M}$ is a retraction mapping to $\mathcal{M}$ and $\alpha > 0$ is a step size.

Let ${\sim}$ be an equivalence relation on $\mathcal{M}$ and consider the quotient manifold $\mathcal{M}/{\sim}$ where each point is an equivalence class $[X] = \{ Y \in \mathcal{M} \mid Y \sim X \}$.
In this setting, $T_X\mathcal{M}$ decomposes as
\begin{equation}
    T_X\mathcal{M} = V_X\mathcal{M} \oplus H_X\mathcal{M},
\end{equation}
where $\oplus$ is the direct sum.
The vertical space $V_X\mathcal{M}$ is the tangent space of the equivalence class $[X]$ at $X$, and the horizontal space $H_X\mathcal{M}$ is its orthogonal complement with respect to the metric, i.e.,
\begin{equation}\label{eq:vertical_horizontal_spaces}
    V_X\mathcal{M} = T_X [X], \qquad
    H_X\mathcal{M} = \{ \xi_X \in T_X \mathcal{M} \mid g_X(\xi_X, \zeta_X) = 0, \forall \zeta_X \in V_X\mathcal{M}\}.
\end{equation}
Although $\mathcal{M}/{\sim}$ is abstract since each point is an equivalence class, each tangent vector $\xi_{[X]} \in T_{[X]}(\mathcal{M}/{\sim})$ can be represented by a unique tangent vector $\xi_X \in H_X\mathcal{M}$ satisfying 
\begin{equation}\label{eq:horizontal_lift}
    \mathrm{D} \pi(X) [\xi_X] = \xi_{[X]},
\end{equation}
where $\pi : \mathcal{M} \rightarrow \mathcal{M}/{\sim}$ is the natural projection.
This unique tangent vector is called the horizontal lift of $\xi_{[X]}$ at $X$.
If the metric $g_X$ on $\mathcal{M}$ is invariant under the equivalence relation, i.e., $g_{X}(\xi_{X}, \zeta_{X}) = g_{Y}(\xi_{Y}, \zeta_{Y})$ whenever $X \sim Y$, we can define a metric $g_{[X]}$ on $\mathcal{M}/{\sim}$ as follows:
\begin{equation}\label{eq:def_riemannian_quotient}
    g_{[X]}(\xi_{[X]}, \zeta_{[X]}) \coloneqq g_{X}(\xi_{X}, \zeta_{X}),
\end{equation}
where $\xi_{X}, \zeta_{X} \in H_X\mathcal{M}$ are the horizontal lifts of $\xi_{[X]}, \zeta_{[X]} \in T_{[X]}(\mathcal{M}/{\sim})$ at $X$, respectively.
Endowed with this invariant metric, $\mathcal{M}/{\sim}$ is called a Riemannian quotient manifold, and the natural projection $\pi$ is a Riemannian submersion.
Let the function $f$ on $\mathcal{M}$ be invariant under the equivalence relation, i.e., $f(X) = f(Y)$ whenever $X \sim Y$.
This induces a well-defined function $\tilde{f} : \mathcal{M}/{\sim} \rightarrow \R$, which satisfies $f = \tilde{f} \circ \pi$.
In this case, $\operatorname{grad}_{X} f$ lies in $H_X\mathcal{M}$ and is the horizontal lift of $\operatorname{grad}_{[X]} \tilde{f}$ at $X$.
Let the retraction $\mathcal{R}_{X}$ preserve the equivalence relation, i.e., $\mathcal{R}_{X}(\xi_{X}) \sim \mathcal{R}_{Y} (\xi_{Y})$ whenever $X \sim Y$, where $\xi_{X}$ and $\xi_{Y}$ are the horizontal lifts of $\xi_{[X]}$ at $X$ and $Y$, respectively.
Then, the $t$-th update for \cref{eq:obj_func_riemannian}, described in \cref{eq:update_riemannian}, respects the geometry of $\mathcal{M}/{\sim}$ and rigorously corresponds to the $t$-th update for minimizing $\tilde{f}$ with respect to $[X]$ on $\mathcal{M}/{\sim}$.

\section{Riemannian Preconditioning for LoRA}\label{sec:LoRA_on_Quotient_Manifold}
In this section, we review the pioneered Riemannian preconditioning study for LoRA \cite{zhang2024riemann}, which is based on the classical low-rank matrix completion problem \cite{mishra2012riemanniangeometry}.

The loss function $\mathcal{L}$ in \cref{eq:lora} is invariant under the equivalence relation
\begin{equation}\label{eq:invariant_GL(r)}
    (B, A) \sim (BG^{-1}, AG^\top),
\end{equation}
in the sense that $\mathcal{L}\,( W_0 + BA^\top) = \mathcal{L}\,( W_0 + BG^{-1}(AG^\top)^\top)$ for all $G \in \mathrm{GL}(r)$.
This invariance implies that the original search space $\mathcal{X}$ admits redundant directions along which the value of $\mathcal{L}$ remains unchanged.
To eliminate this redundancy, \citet{zhang2024riemann} have proposed to exploit the quotient manifold 
\begin{equation}\label{eq:quotient_manifold}
    \mathcal{X}/\mathrm{GL}(r),
\end{equation}
where each point is an equivalence class
\begin{equation}\label{eq:equivalence_class}
    [(B,A)] = \{(BG^{-1}, AG^\top) \mid G \in \operatorname{GL}(r) \}.
\end{equation}
On this manifold, all matrices in $[(B,A)]$ are identified and thus the redundant directions are eliminated.
To respect the geometry of this manifold, the original search space $\mathcal{X}$ is endowed with a Riemannian metric that is invariant under the equivalence relation \cref{eq:invariant_GL(r)}.
Following \cite{mishra2012riemanniangeometry}, \citet{zhang2024riemann} adopted the well-known invariant metric
\begin{equation}\label{eq:metric}
    g_{(B,A)} \big( (\xi_B,\xi_A), (\zeta_B,\zeta_A) \big)
    = \tr\left( A^\top A \xi_B^\top \zeta_B \right) + \tr\left( B^\top B \xi_A^\top \zeta_A \right),
\end{equation}
where $(\zeta_B,\zeta_A), (\xi_B,\xi_A) \in T_{(B,A)}\mathcal{X}$.
In this setting, the horizontal space is given by
\begin{equation}\label{eq:horizon_space}
    H_{(B,A)}\mathcal{X} = \left\{ (\xi_B, \xi_A) \in T_{(B,A)} \mathcal{X} \mid B^\top \xi_B A^\top A = B^\top B \xi_A^\top A \right\},
\end{equation}
see \cite[(13)]{mishra2012riemanniangeometry}.
By solving \cref{eq:def_rieman_grad} with this invariant metric, we obtain the Riemannian gradient $(\operatorname{grad}_{B} \mathcal{L}, \operatorname{grad}_{A} \mathcal{L})$, which lies in the horizontal space $H_{(B,A)}\mathcal{X}$, as follows:
\begin{equation}\label{eq:grad_quotient}
    \operatorname{grad}_{B} \mathcal{L} = (\nabla_{B} \mathcal{L}) (A^\top A)^{-1}, \qquad
    \operatorname{grad}_{A} \mathcal{L} = (\nabla_{A} \mathcal{L}) (B^\top B)^{-1},
\end{equation}
where $\nabla_{B}\mathcal{L}$ and $\nabla_{A}\mathcal{L}$ denote the Euclidean gradients of $\mathcal{L}$ with respect to $B$ and $A$, respectively.
Given a widely used simple retraction $\mathcal{R}_{(B,A)}(\xi_{B},\xi_{A}) = (B + \xi_{B}, A +\xi_{A})$ which preserves the equivalence relation \cref{eq:invariant_GL(r)}, plugging \cref{eq:grad_quotient} into \cref{eq:update_riemannian} yields the $t$-th update of the Riemannian gradient descent method for \cref{eq:lora} that respects the geometry of $\mathcal{X}/\mathrm{GL}(r)$:
\begin{equation}\label{eq:update_quotient}
    B_{t+1} = B_t - \alpha (\nabla_{B_t} \mathcal{L}) (A_t^\top A_t)^{-1}, \qquad
    A_{t+1} = A_t - \alpha (\nabla_{A_t} \mathcal{L}) (B_t^\top B_t)^{-1},
\end{equation}
where $\nabla_{B_t}\mathcal{L}$ and $\nabla_{A_t}\mathcal{L}$ are evaluated at $(B_t, A_t)$.
Considering the plain gradient descent update, i.e., $B_{t+1} = B_t - \alpha (\nabla_{B_t} \mathcal{L})$ and $A_{t+1} = A_t - \alpha (\nabla_{A_t} \mathcal{L})$, the above update can be interpreted as the gradient descent update with preconditioning \cref{eq:preconditioning_quotient}.
For details, see \cite{mishra2012riemanniangeometry,zhang2024riemann}.

\section{Proposed Method: Riemannian Preconditioning for LoRA with New Metric}\label{sec:proposed_method}
Instead of the classical metric \cref{eq:metric}, we propose to endow the search space $\mathcal{X}$ with a new metric
\begin{equation}\label{eq:proposed_metric}
    g_{(B,A)} \big( (\xi_B,\xi_A), (\zeta_B,\zeta_A) \big)
    = \tr\left( A^\top A \xi_B^\top \left( I + P_B \right) \zeta_B \right)
    + \tr\left( B^\top B \xi_A^\top \left( I + P_A \right) \zeta_A \right).
\end{equation}
%
We first show that this metric is invariant under the equivalence relation \cref{eq:invariant_GL(r)}, with reference to \cite[Propositions 3.1 and 3.2]{meyer2011factorization}.
\begin{lemma}\label{lemma:horizontal_space}
    The manifold $\mathcal{X}$ endowed with the metric \cref{eq:proposed_metric} induces the horizontal space \cref{eq:horizon_space}.
\end{lemma}
\begin{proposition}\label{prop:our_metric_is_invariant}
    The manifold $\mathcal{X}$ is endowed with the metric \cref{eq:proposed_metric}.
    Let $\xi_{[(B,A)]}$ be a tangent vector to the quotient manifold $\mathcal{X}/\mathrm{GL}(r)$ at $[(B,A)]$.
    The horizontal lifts of $\xi_{[(B,A)]}$ at $(B,A)$, denoted as $(\xi_B, \xi_A)$, and at $(BG^{-1}, AG^{\top})$, denoted as $(\xi_{BG^{-1}}, \xi_{AG^{\top}})$, have a relation $(\xi_{BG^{-1}}, \xi_{AG^{\top}}) = (\xi_{B} G^{-1}, \xi_{A} G^{\top})$.
    Then, the metric \cref{eq:proposed_metric} is invariant under the equivalence relation \cref{eq:invariant_GL(r)}.
\end{proposition}
These proofs are provided in \cref{ap_sec:proof_horizontal_space,ap_sec:proof_invariant}, respectively.
From \cref{prop:our_metric_is_invariant}, by solving \cref{eq:def_rieman_grad} with our invariant metric \cref{eq:proposed_metric}, we obtain the Riemannian gradient $(\operatorname{grad}_{B} \mathcal{L}, \operatorname{grad}_{A} \mathcal{L})$, which lies in the horizontal space $H_{(B,A)}\mathcal{X}$, as follows:
\begin{equation}\label{eq:grad_our}
    \operatorname{grad}_{B} \mathcal{L} = ( I - \tfrac{1}{2} P_B ) (\nabla_{B} \mathcal{L}) (A^\top A)^{-1}, \quad
    \operatorname{grad}_{A} \mathcal{L} = ( I - \tfrac{1}{2} P_A ) (\nabla_{A} \mathcal{L}) (B^\top B)^{-1}.
\end{equation}
This derivation is provided in \cref{ap_subsec:deviation_our_riemannian_grad}.
Compared to \cref{eq:grad_quotient}, this is multiplied from the left by $(I - \tfrac{1}{2} P_{B})$ and $(I - \tfrac{1}{2} P_{A})$, respectively.
These matrices can be rewritten as $\left(I - \tfrac{1}{2} P_{B}\right) = \tfrac{1}{2} P_{B} + Q_B$ and $\left(I - \tfrac{1}{2} P_{A}\right) = \tfrac{1}{2} P_{A} + Q_A$, where $Q_B \coloneqq (I - P_{B})$ and $Q_A \coloneqq (I - P_{A})$ are the complements of $P_{B}$ and $P_{A}$, which project onto the orthogonal column spaces of $B$ and $A$, denoted by $\operatorname{col}(B)^\perp$ and $\operatorname{col}(A)^\perp$.
Considering the identity maps are given by $P_{B} + Q_B$ and $P_{A} + Q_A$, our Riemannian gradient \cref{eq:grad_our} can be interpreted as a subspace-guided extension of \cref{eq:grad_quotient} that attenuates by half the gradient directions along $\operatorname{col}(B)$ and $\operatorname{col}(A)$, while preserving those along $\operatorname{col}(B)^\perp$ and $\operatorname{col}(A)^\perp$.

Given the same retraction in \cref{sec:LoRA_on_Quotient_Manifold}, plugging \cref{eq:grad_our} into \cref{eq:update_riemannian} yields the $t$-th update of the Riemannian gradient descent method for \cref{eq:lora} that respects the geometry of $\mathcal{X}/\mathrm{GL}(r)$:
\begin{equation}\label{eq:update_our}
    B_{t+1} = B_t - \alpha ( I - \tfrac{1}{2} P_{B_t} ) (\nabla_{B_t} \mathcal{L}) (A_t^\top A_t)^{-1},\,\,
    A_{t+1} = A_t - \alpha ( I - \tfrac{1}{2} P_{A_t} ) (\nabla_{A_t} \mathcal{L}) (B_t^\top B_t)^{-1}.
\end{equation}
This update can be interpreted as the gradient descent update with preconditioning~\cref{eq:preconditioning_our}.

\subsection{An Efficient Implementation of Our Preconditioning}\label{sec:implementation}
The computation of the conventional preconditioning for $\nabla_{B} \mathcal{L}$ in \cref{eq:preconditioning_quotient}, i.e., $(\nabla_{B} \mathcal{L}) (A^\top A)^{-1}$, requires $O(nr^2)$ time to form $A^\top A$, $O(r^3)$ for its inversion, and $O(mr^2)$ for the remaining matrix multiplication, resulting in a total time complexity of $O((m+n)r^2 + r^3)$.
Surprisingly, our preconditioning for $\nabla_{B} \mathcal{L}$ in \cref{eq:preconditioning_our}, i.e., $(I - \tfrac{1}{2} P_{B}) (\nabla_{B} \mathcal{L}) (A^\top A)^{-1}$, admits an efficient implementation with the same time complexity, even though $(I - \tfrac{1}{2} P_{B})$ is added.
A naive implementation requires $O(m^2r)$ time to explicitly form $(I - \tfrac{1}{2} P_{B})$.
However, this preconditioning can be rewritten as $(\nabla_{B} \mathcal{L}) (A^\top A)^{-1} - \tfrac{1}{2} B (B^\top B)^{-1} (B^\top (\nabla_{B} \mathcal{L}) (A^\top A)^{-1} )$, which avoids forming any $m \times m$ matrices.
The second term requires $O((m+n)r^2)$ time to form $B^\top (\nabla_{B} \mathcal{L}) (A^\top A)^{-1}$, $O(mr^2)$ for $B^\top B$, $O(r^3)$ for its inversion, and $O(mr^2)$ for the remaining matrix multiplication, resulting in the same total time complexity of $O((m+n)r^2 + r^3)$.
Note that this bound simplifies to $O((m+n)r^2)$ since $r \ll \min(m,n)$.
A similar discussion holds in preconditioning for $\nabla_{A} \mathcal{L}$.


\section{Theoretical Insights}\label{sec:theoretical_insights}
We here investigate how LoRA with our preconditioning \cref{eq:preconditioning_our} affect the weight matrix $W$ at each iteration $t$.
Specifically, let $W_{t}$ denote the $t$-th weight matrix, we analyze (i) the $t$-th weight update from $W_{t}$ to $W_{t+1}$ and (ii) the $t$-th updated weight matrix $W_{t+1}$.
For this purpose, we first show the $t$-th weight update via full fine-tuning and LoRA with different preconditioning:
the conventional preconditioning \cref{eq:preconditioning_quotient} and
our preconditioning \cref{eq:preconditioning_our}.
Hereafter, we ignore terms of order $O(\alpha^2)$ by assuming that the step size $\alpha$ is sufficiently small, and use the chain rule: $\nabla_{B} \mathcal{L} = (\nabla_{W} \mathcal{L}) A$ and $\nabla_{A} \mathcal{L} = (\nabla_{W} \mathcal{L})^\top B$, where $\nabla_{W} \mathcal{L}$ denote the Euclidean gradient of $\mathcal{L}$ with respect to $W$. 
Also, we assume the scaling factor for $BA^\top$ as $s = 1$, no dropout, and no weight decay.

Let $W_{t+1}^{\operatorname{FFT}}$ denote the $t$-th updated weight matrix of full fine-tuning.
Then, the $t$-th weight update via full fine-tuning is given by $W_{t+1}^{\operatorname{FFT}} = W_t - \alpha \nabla_{W_t} \mathcal{L}$ using the plain gradient descent update, where $\nabla_{W_t} \mathcal{L}$ is evaluated at $W_t$.
Considering $\nabla_{W_t} \mathcal{L} = (P_{B_t} + Q_{B_t}) (\nabla_{W_t} \mathcal{L}) (P_{A_t} + Q_{A_t})$, this update can be rewritten as
\begin{equation}\label{eq:W_(t+1)_FFT}
    W_{t+1}^{\operatorname{FFT}}
    = W_t - \alpha P_{B_t} (\nabla_{W_t} \mathcal{L}) - \alpha (\nabla_{W_t} \mathcal{L}) P_{A_t}
    + \alpha P_{B_t} (\nabla_{W_t} \mathcal{L}) P_{A_t} - \alpha Q_{B_t} (\nabla_{W_t} \mathcal{L}) Q_{A_t}.
\end{equation}
Thus, the $t$-th weight update via full fine-tuning follows the gradient of full fine-tuning $\nabla_{W_t} \mathcal{L}$ that can be decomposed into components involving projections onto the subspaces induced by $B_t$ and $A_t$.

On the other hand, let $W_{t+1}^{\operatorname{LoRA}}$ denote the $t$-th updated weight matrix of LoRA, and then the $t$-th weight update via LoRA is given by $W_{t+1}^{\operatorname{LoRA}} = W_0 + B_{t+1} A_{t+1}^\top$ because of the LoRA parameterization~\cref{eq:modeling_lora}.
By using the plain gradient descent update for $B_{t+1}$ and $A_{t+1}$, this update can be rewritten as
\begin{align}
    W_{t+1}^{\operatorname{LoRA}}
    &= W_0 + (B_{t} - \alpha \nabla_{B_t} \mathcal{L}) (A_{t} - \alpha \nabla_{A_t} \mathcal{L})^\top \label{eq:W_(t+1)_LoRA_before_expand}\\
    &\approx W_t - \alpha B_t B_t^\top (\nabla_{W_t} \mathcal{L}) - \alpha (\nabla_{W_t} \mathcal{L}) A_t A_t^\top,\label{eq:W_(t+1)_LoRA}
\end{align}
where we used the relation $W_{t} = W_{0} + B_t A_t^\top$.
This update follows the gradient of full fine-tuning $\nabla_{W_t} \mathcal{L}$ projected approximately onto the column spaces of $B_t$ and $A_t$, as $B_t B_t^\top (\nabla_{W_t} \mathcal{L})$ and $(\nabla_{W_t} \mathcal{L}) A_t A_t^\top$, respectively.

Applying the conventional preconditioning \cref{eq:preconditioning_quotient} to \cref{eq:W_(t+1)_LoRA_before_expand} yields
\begin{equation}\label{eq:W_(t+1)_LoRACP}
    W_{t+1}^{\operatorname{LoRA-CP}}
    \approx W_t - \alpha P_{B_t} (\nabla_{W_t} \mathcal{L}) - \alpha (\nabla_{W_t} \mathcal{L}) P_{A_t},
\end{equation}
where $W_{t+1}^{\operatorname{LoRA-CP}}$ denotes $W_{t+1}^{\operatorname{LoRA}}$ with \cref{eq:preconditioning_quotient}.
This result has been reported in \cite[Sec. 3]{zhang2024riemann}, and we provide this deviation in \cref{ap_sec:derivation_of_W_LoRACP}.
Compared to \cref{eq:W_(t+1)_LoRA}, this update follows the gradient of full fine-tuning $\nabla_{W_t} \mathcal{L}$ projected onto the column spaces of $B_t$ and $A_t$ by $P_{B_t}$ and $P_{A_t}$, respectively, and it matches the first three terms in \cref{eq:W_(t+1)_FFT}.

Applying our preconditioning \cref{eq:preconditioning_our} to \cref{eq:W_(t+1)_LoRA_before_expand} yields
\begin{equation}\label{eq:W_(t+1)_our}
    W_{t+1}^{\operatorname{LoRA-OP}}
    \approx W_t - \alpha P_{B_t} (\nabla_{W_t} \mathcal{L}) - \alpha (\nabla_{W_t} \mathcal{L}) P_{A_t} 
    + \alpha P_{B_t} (\nabla_{W_t} \mathcal{L}) P_{A_t},
\end{equation}
where $W_{t+1}^{\operatorname{LoRA-OP}}$ denotes $W_{t+1}^{\operatorname{LoRA}}$ with \cref{eq:preconditioning_our}.
This derivation is provided in \cref{ap_sec:derivation_of_W_our}.
Compared to \cref{eq:W_(t+1)_LoRACP}, this update successfully incorporates the term $P_{B_t} (\nabla_{W_t} \mathcal{L}) P_{A_t}$ and matches the first four terms in \cref{eq:W_(t+1)_FFT}.
For further analysis, we show the following proposition.
\begin{proposition}\label{prop:projection}
    The first-order weight changes allowed by the LoRA parameterization \cref{eq:modeling_lora} form the subspace $\mathcal{S}_{(B,A)} = \{ \mathrm{d} W \in \R^{m \times n} \mid Q_B (\mathrm{d} W) Q_A = O\}$.
    Then, the orthogonal projection of $X\in\R^{m \times n}$ onto $\mathcal{S}_{(B,A)}$, denoted as $\operatorname{Proj}_{\mathcal{S}_{(B,A)}}(X)$, in terms of the Frobenius norm is given by $\operatorname{Proj}_{\mathcal{S}_{(B,A)}}(X) = P_B X + X P_A - P_B X P_A$.
\end{proposition}
This proof is provided in \cref{ap_sec:proof_projection}.
Then, we can define the $t$-th weight update via the projected full fine-tuning onto $\mathcal{S}_{(B_t,A_t)}$ as follows:
\begin{align}
    W_{t+1}^{\operatorname{FFT-Proj}} 
    &= W_t - \alpha \operatorname{Proj}_{\mathcal{S}_{(B_t,A_t)}}(\nabla_{W_t} \mathcal{L}) \\
    &= W_t - \alpha P_{B_t} (\nabla_{W_t} \mathcal{L}) - \alpha (\nabla_{W_t} \mathcal{L}) P_{A_t} 
    + \alpha P_{B_t} (\nabla_{W_t} \mathcal{L}) P_{A_t},
\end{align}
where $W_{t+1}^{\operatorname{FFT-Proj}}$ denotes $W_{t+1}^{\operatorname{FFT}}$ with the projection.
Obviously, this is identical to \cref{eq:W_(t+1)_our}.
Therefore, the $t$-th weight update via LoRA with our preconditioning, described in \cref{eq:W_(t+1)_our}, follows the direction closest to the gradient of full fine-tuning $\nabla_{W_t} \mathcal{L}$ within the subspace of first-order weight changes allowed by the LoRA parameterization, i.e., $\mathcal{S}_{(B_t,A_t)}$, compared with \cref{eq:W_(t+1)_LoRA,eq:W_(t+1)_LoRACP}.
Moreover, we show the following proposition.
\begin{proposition}\label{prop:resulting_full_weight_matrix}
    The $t$-th updated weight matrix of LoRA with \cref{eq:preconditioning_our}, i.e., $W_{t+1}^{\operatorname{LoRA-OP}}$, satisfies
    \begin{equation}
        \left\| W_{t+1}^{\operatorname{FFT}} - W_{t+1}^{\operatorname{LoRA-OP}} \right\|_F \leq \left\| W_{t+1}^{\operatorname{FFT}} - W_{t+1}^{\operatorname{LoRA}} \right\|_F,
    \end{equation}
    and its inequality is strict whenever $( B_t B_t^\top - P_{B_t}) (\nabla_{W_t} \mathcal{L}) + (\nabla_{W_t} \mathcal{L}) (A_t A_t^\top - P_{A_t})
    + P_{B_t} (\nabla_{W_t} \mathcal{L}) P_{A_t} \neq O$.
    Also, it satisfies 
    \begin{equation}
        \left\| W_{t+1}^{\operatorname{FFT}} - W_{t+1}^{\operatorname{LoRA-OP}} \right\|_F \leq \left\| W_{t+1}^{\operatorname{FFT}} - W_{t+1}^{\operatorname{LoRA-CP}} \right\|_F,
    \end{equation}
    and its inequality is strict whenever $P_{B_t} (\nabla_{W_t} \mathcal{L}) P_{A_t} \neq O$.
\end{proposition}
This proof is provided in \cref{ap_sec:proof_weight_comparison}.
\cref{prop:resulting_full_weight_matrix} indicates that our preconditioning \cref{eq:preconditioning_our} makes the $t$-th updated weight matrix of LoRA closer to that of full fine-tuning in terms of the Frobenius norm, compared with using \cref{eq:preconditioning_quotient} and no preconditioning.
We also empirically validate this proposition in \cref{ap_sec:experiment_for_proposition}.
From the above discussions, our preconditioning makes LoRA better approximates full fine-tuning, and therefore is plausible to further improve the performance of LoRA.


\section{Experiments}\label{sec:experiment}
To validate the effectiveness of our preconditioning \cref{eq:preconditioning_our} for LoRA, we conducted LoRA fine-tuning experiments on natural language generation, image classification, and commonsense reasoning tasks, following the existing LoRA studies \cite{hu2022lowrank,zhang2024riemann,wang2025lorapro,park2025stiefel,li2025stiefel}.
All experiments were performed in PyTorch \cite{pytorch} on a single H100 GPU with 80 GB VRAM.
We downloaded the pre-trained models using Hugging Face Transformers library \cite{HuggingFaceTransformer}.

\subsection{Methods for Comparison}
In the experiments, we implemented the two widely used optimization methods, the stochastic gradient descent (SGD) \cite{sgd} and the AdamW \cite{loshchilov2018adamw} methods, with different preconditioning:
(i) no preconditioning (i.e., the original SGD and AdamW methods), 
(ii) the conventional preconditioning \cref{eq:preconditioning_quotient}, and
(iii) our preconditioning~\cref{eq:preconditioning_our}.
For each method, preconditioning was applied at each gradient step.
Note that, as in prior Riemannian optimization approaches for LoRA~\cite{hu2022lowrank, zhang2024riemann,sun2025strongmixture,li2025stiefel,park2025stiefel}, we did not employ standard geometric operations such as the exponential map and parallel transport in the AdamW methods, since incorporating these operations would be computationally prohibitive for large-scale models.
Instead, we performed gradient preconditioning and then followed the standard AdamW scheme.
Following \cite{zhang2024riemann}, we replaced the computation of $(B^\top B)^{-1}$ and $(A^\top A)^{-1}$ with $(B^\top B + \lambda I)^{-1}$ and $(A^\top A + \lambda I)^{-1}$ where a small $\lambda > 0$ addresses the case when either $(B^\top B)^{-1}$ or $(A^\top A)^{-1}$ is not invertible.
For all experiments, we set $\lambda = 1.0 \times 10^{-6}$.
In the AdamW methods, we set the exponential decay rates $(\beta_1, \beta_2)$ to $(0.9, 0.999)$ for the original AdamW and $(0.7, 0.8)$ for others following the setup in \cite{zhang2024riemann}.
In our experiments, we tuned the learning rate of each method and followed other hyperparameter settings in previous works \cite{hu2022lowrank, zhang2024riemann, park2025stiefel, li2025stiefel, sun2025strongmixture}.
However, unlike these previous works that use only train and test sets in each dataset, we prepared train, validation, and test sets in each dataset, and then tuned the learning rates of each method on the validation set for a fair comparison.
We summarize the details of the datasets and hyperparameters in \cref{ap_sec:datasets,ap_sec:hyperparameter}.

\subsection{Natural Language Generation}\label{subsec:NLG}
Following \cite{hu2022lowrank,zhang2024riemann}, we conducted LoRA fine-tuning for the GPT-2 medium model \cite{radford2019gpt2} on E2E natural language generation challenge task \cite{novikova2017e2e}, which assesses the fine-tuned model's ability to generate fluent and semantically accurate sentences from meaning representations.
We followed the same experimental setup as in \cite{zhang2024riemann}, except that we tuned the learning rate over $\{0.02, 0.04, \dots, 0.30\}$ for SGD and $\{0.0002, 0.0004, \dots, 0.0030\}$ for AdamW methods on the validation set.
For all methods, we inserted the low-rank matrix $BA^\top$ into the query and value matrices of the attention layers in the GPT-2 model.
We also fixed the rank to $r = 4$, the scaling factor of $BA^\top$ to $32$, the batch size to $8$, the dropout to $0.1$, and the train epoch to $5$ with a linear learning-rate schedule.
We set the weight decay to $0.01$ for the SGD and AdamW methods.
We evaluated all methods by five metrics, BLEU \cite{papineni-etal-2002-bleu}, NIST \cite{NIST}, METEOR \cite{banerjee-lavie-2005-meteor}, ROUGE-L \cite{lin-2004-rouge}, and CIDEr \cite{cider}, on the test set.

\begin{table*}[t]
    \caption{
        Scores in LoRA fine-tuning for the GPT-2 medium model on E2E natural language generation challenge task with different preconditioning.
        ``w/'' denotes ``with''.
        Bold and underlined values indicate the highest and second-highest scores, respectively.
    }\label{table:gpt2}
    \centering
    \begin{tabular}[t]{lcccccccccc}
        \toprule
        Method
        & BLEU & NIST & METEOR & ROUGE-L & CIDEr \\
        \midrule
        SGD
        & 0.6698 & 8.5505 & 0.4479 & 0.6908 & 2.2996 \\
        w/ \cref{eq:preconditioning_quotient}
        & 0.6852 & 8.6371 & 0.4641 & 0.7097 & 2.4367 \\
        w/ \cref{eq:preconditioning_our} (Ours)
        & 0.6944 & 8.7449 & 0.4632 & 0.7127 & 2.5082 \\        
        \midrule
        \midrule
        AdamW
        & \underline{0.7020} & \underline{8.8433} & \underline{0.4676} & \underline{0.7193} & {\bf 2.5500} \\
        w/ \cref{eq:preconditioning_quotient}
        & 0.7003 & 8.8170 & 0.4671 & 0.7181 & 2.5312 \\
        w/ \cref{eq:preconditioning_our} (Ours)
        & {\bf 0.7047} & {\bf 8.8527} & {\bf 0.4697} & {\bf 0.7209} & \underline{2.5454} \\
        \bottomrule
    \end{tabular}
\end{table*}

\begin{table*}[t]
    \caption{
        Top-1 test accuracy in LoRA fine-tuning for the CLIP ViT-B/32 model on the seven image classification datasets with different preconditioning.
        ``Ave.'' denotes ``Average''.
    }\label{table:vit}
    \setlength{\tabcolsep}{5pt} 
    \centering
    \begin{tabular}[t]{lcccccccccc}
        \toprule
        Method
        & Cars & DTD & EuroSAT & GTRSB & RESISC45 & SUN397 & SVHN & Ave.\\
        \midrule
        AdamW
        & 71.82 & 70.53 & \underline{98.35} & 97.84 & 94.48 & 72.04 & 96.79 & 85.98\\
        w/ \cref{eq:preconditioning_quotient}
        & {\bf 72.14} & 71.28 & 98.19 & \underline{98.32} & \underline{94.49} & {\bf 72.64} & \underline{96.83} & \underline{86.27}\\
        w/ \cref{eq:preconditioning_our} (Ours)
        & \underline{72.03} & {\bf 71.86} & {\bf 98.52} & {\bf 98.54} & {\bf 94.70} & \underline{72.57} & {\bf 96.92} & {\bf 86.45}\\
        \bottomrule
    \end{tabular}
\end{table*}
\begin{table*}[t]
    \caption{
        Top-1 test accuracy in LoRA fine-tuning for the LLaMA2-7B model on the eight commonsense reasoning datasets with different preconditioning.
    }\label{table:cr}
    \setlength{\tabcolsep}{5pt} 
    \centering
    \begin{tabular}[t]{lcccccccccc}
        \toprule
        Method
        & BoolQ & PIQA & SIQA & HellaS & WinoG & ARC-e & ARC-c & OBQA & Ave.\\
        \midrule
        AdamW
        & 73.33 & 83.41 & {\bf 79.38} & 92.36 & \underline{83.74} & 85.77 & 72.87 & \underline{82.60} & 81.68 \\
        w/ \cref{eq:preconditioning_quotient}
        & {\bf 74.31} & {\bf 84.17} & 77.99 & {\bf 93.12} & 83.43 & \underline{86.28} & \underline{74.23} & 82.20 & \underline{81.97} \\
        w/ \cref{eq:preconditioning_our} (Ours)
        & \underline{73.55} & \underline{83.62} & \underline{79.02} & \underline{92.95} & {\bf 85.48} & {\bf 86.41} & {\bf 74.32} & {\bf 83.40} & {\bf 82.34} \\
        \bottomrule
    \end{tabular}
\end{table*}

\Cref{table:gpt2} shows the results.
We observed that our SGD and AdamW methods outperform other SGD and AdamW methods, respectively, across almost all metrics.
Especially, our AdamW method achieves the highest scores in four metrics and the second-highest in the remaining one.
These results validate the effectiveness of our preconditioning for LoRA on the language generation task.

\subsection{Image Classification}\label{subsec:image_classification}
Following \cite{li2025stiefel,chenghao2025greatagain}, we conducted LoRA fine-tuning for the CLIP ViT-B/32 \cite{dosovitskiy2021vit} on the seven image classification datasets, such as Cars \cite{Cars}, DTD \cite{DTD}, EuroSAT \cite{EuroSAT}, GTSRB \cite{GTSRB}, RESISC45 \cite{RESISC45}, SUN397 \cite{SUN397}, and SVHN \cite{SVHN}.
In these datasets, we measured the top-1 test accuracy.
We tuned the learning rate over $\{0.0001, 0.0002, \dots, 0.001\}$ for each method on the validation set in each dataset separately.
Because \citet{zhang2024whyneedadam} reported that ViT-base models heavily depend on the Adam-based optimizers, we here tested only the AdamW methods with preconditioning.
For all methods, we inserted the low-rank matrix $BA^\top$ into the query and value matrices of the attention layers in the vision model of CLIP ViT-B/32.
We also fixed the rank to $r = 16$, the scaling factor of $BA^\top$ to $16$, the batch size to $64$, the weight decay to zero, the dropout to $0.05$, and the train epoch to $5$ with a linear learning-rate schedule.

\Cref{table:vit} shows the results.
We observed that our AdamW method achieves either the highest or second-highest top-1 test accuracy across datasets, leading to the highest average.
These results validate the effectiveness of our preconditioning for LoRA on the image classification task.

\subsection{Commonsense Reasoning}\label{subsec:CR}
Following \cite{li2025stiefel,park2025stiefel,chenghao2025greatagain}, we conducted LoRA fine-tuning for the LLaMA2-7B \cite{touvron2023llama2openfoundation} on the commonsense reasoning benchmark, which assesses the reasoning capabilities of the fine-tuned model across the eight datasets, such as BoolQ \cite{clark-etal-2019-boolq}, PIQA \cite{piqa}, SIQA \cite{sap-etal-2019-siqa}, HellaS \cite{zellers-etal-2019-hellaswag}, WinoG \cite{winogrande}, ARC-e \cite{arc-e_and_arc_c}, ARC-c \cite{arc-e_and_arc_c}, and OBQA \cite{mihaylov-etal-2018-obqa}.
In these datasets, we measured the top-1 test accuracy.
For fair comparison, we first split a combined dataset of all train sets, i.e., Commonsense170K, into the train and validation sets, tuned the learning rate on the validation set, and then evaluated the fine-tuned model on the test set in each dataset.
In this experiment, we tested the Adam methods with preconditioning.
We tuned the learning rate over $\{0.0001, 0.0002, \dots, 0.001\}$ for each method on the validation set, similar to \cref{subsec:image_classification}.
For all methods, we inserted the low-rank matrix $BA^\top$ into the query, key, and value projections of the attention layers and the up and down projections of the feed-forward layers.
We also fixed the rank to $r = 32$, the scaling factor of $BA^\top$ to $64$, the batch size to $16$, the weight decay to zero, the dropout to $0.05$, and the train epoch to $3$ with a linear learning-rate schedule.

\Cref{table:cr} shows the results.
We observed that our AdamW method achieves either the highest or second-highest top-1 test accuracy across datasets, leading to the highest average.
These results validate the effectiveness of our preconditioning for LoRA on the commonsense reasoning task.

\subsection{Empirical Validation: Training Time Comparison}\label{subsec:training_time}
To validate the time complexity analysis in \cref{sec:implementation} empirically, we investigate how preconditioning affects training time.
For this purpose, we compared training times of the CLIP ViT-B/32 model on the seven image classification datasets when using the AdamW methods with different preconditioning.
The experimental setup was the same as that described in \cref{subsec:image_classification}.
The difference in training times mainly stems from dataset size and the computational cost of preconditioning.
\Cref{fig:computation_time} shows the results.
Each method shows the mean and standard deviation of training times over all learning rates we used for validation.
We observed that, although there are some differences due to constant time factors, the results are almost the same across all methods, including the AdamW with our preconditioning \cref{eq:preconditioning_our} as it admits the efficient implementation described in \cref{sec:implementation}.
Consequently, these results validate the time complexity analysis in \cref{sec:implementation} empirically and show the computational efficiency of our preconditioning \cref{eq:preconditioning_our}.

\begin{figure}
    \centering
    \includegraphics[width=1.0\linewidth]{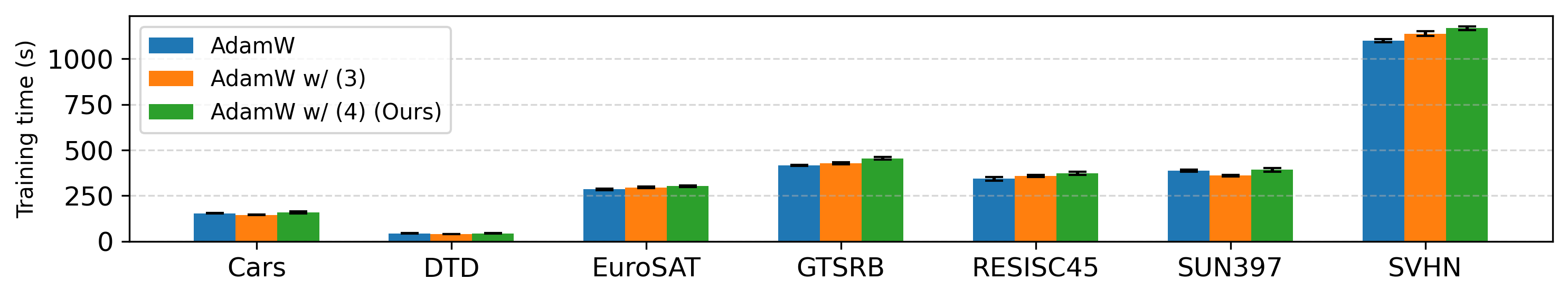}
    \vspace{-6mm}
    \caption{
        Training time comparison of the CLIP ViT-B/32 model on the seven image classification datasets when using the AdamW methods with different preconditioning.
        }
    \label{fig:computation_time}
\end{figure}

\section{Discussions and Limitations}\label{sec:discussion}
Throughout this paper, we confirmed that our preconditioning improves the performance of LoRA.
For further progress, we discuss the following future issues.
(I) As a first step, we tested our preconditioning in the widely used SGD and AdamW optimizers.
Extending our approach to other sophisticated optimizers is a promising direction for future research.
(II) Our method is based on the standard two-factor factorization of LoRA.
Combining other modeling, e.g., three-factor factorization \cite{wu2024mosLoRA,li2025stiefel,chenghao2025greatagain}, with our method is left for future work.
(III) Combining the state-of-the-art initialization strategies for LoRA with our method has excellent potential but is nontrivial because our approach may break those strategy concepts.
We will explore this potential in the future.
(IV) Our method is inspired by the first-order approximation of the weight update of full fine-tuning, which means each update is sufficiently small.
We must consider the higher-order effective and will design more efficient preconditioning approaches in the future.
(V) Testing our method on various models and tasks, e.g., image generation tasks for qualitative visual evaluation, is a promising direction for further assessing its practical utility.

\section{Conclusions}
In this paper, we proposed a Riemannian metric that is specifically tailored to LoRA to close the gap to full fine-tuning at the weight level.
This metric induces new preconditioning, and we theoretically showed that this preconditioning makes LoRA better approximates full fine-tuning compared with the conventional preconditioning and no preconditioning.
Diverse experiments showed the effectiveness and efficiency of our preconditioning in the language and vision domains.
This paper is the first to explore the possibility of tailoring a Riemannian metric to LoRA, paving the way for future followers.

\bibliographystyle{plainnat}
\bibliography{neurips2026}

\newpage
\appendix

\section{Proofs}
\subsection{Proof of \cref{lemma:horizontal_space}}\label{ap_sec:proof_horizontal_space}
With the equivalence relation \cref{eq:invariant_GL(r)}, the vertical space $V_{(B,A)}\mathcal{X}$ is given by 
\begin{equation}\label{ap_eq:vertical_space}
    V_{(B,A)}\mathcal{X} = \left\{ (-B\Omega, A\Omega^\top) \mid \Omega \in \R^{r \times r} \right\}.
\end{equation}
For details, see \cite[Sec. 3]{mishra2012riemanniangeometry}.
Then, the definition \cref{eq:vertical_horizontal_spaces} leads to the horizontal space
\begin{align}
    H_{(B,A)}\mathcal{X} 
    &= \{ (\xi_B,\xi_A) \in T_X \mathcal{X} \mid g_{(B,A)} \big( (\xi_B,\xi_A), (\zeta_B, \zeta_A) \big) = 0,\, \forall (\zeta_B, \zeta_A) \in V_{(B,A)}\mathcal{X} \}\\
    &= \{ (\xi_B,\xi_A) \in T_X \mathcal{X} \mid g_{(B,A)} \big( (\xi_B,\xi_A), (-B\Omega, A\Omega^\top) \big) = 0,\, \forall \Omega \in \R^{r \times r} \}.
\end{align}
With our metric \cref{eq:proposed_metric}, we obtain
\begin{align}
    g_{(B,A)} \big( (\xi_B,\xi_A), (-B\Omega, A\Omega^\top) \big)
    &= - \tr\left( A^\top A \xi_B^\top \left( I + P_B \right) B\Omega \right)
    + \tr\left( B^\top B \xi_A^\top \left( I + P_A \right) A\Omega^\top \right) \\
    &= - 2 \tr\left( A^\top A \xi_B^\top B\Omega \right)
    + 2 \tr\left( B^\top B \xi_A^\top A\Omega^\top \right)\\
    &= 2 \tr\left( \left( - B^\top \xi_B A^\top A + B^\top B \xi_A^\top A \right)^\top \Omega \right),
\end{align}
because $(I + P_B) B = 2B$ and $(I + P_A) A = 2A$.
For the above expression to be zero for all $\Omega$, it is necessary that
\begin{equation}
    B^\top \xi_B A^\top A = B^\top B \xi_A^\top A.
\end{equation}
Thus, we obtain the horizontal space
\begin{equation}
    H_{(B,A)}\mathcal{X} = \left\{ (\xi_B, \xi_A) \in T_{(B,A)} \mathcal{X} \mid B^\top \xi_B A^\top A = B^\top B \xi_A^\top A \right\},
\end{equation}
which is identical to \cref{eq:horizon_space}.

\subsection{Proof of \cref{prop:our_metric_is_invariant}}\label{ap_sec:proof_invariant}
With the uniqueness of the horizontal lifts described in \cref{eq:horizontal_lift}, the horizontal lifts $(\xi_B, \xi_A) \in H_{(B, A)}\mathcal{X}$ and $(\xi_{BG^{-1}}, \xi_{AG^{\top}}) \in H_{(BG^{-1}, AG^{\top})}\mathcal{X}$ have the relationship
\begin{equation}\label{ap_eq:directional_derivative_1}
    \mathrm{D} \pi(B,A) [(\xi_B, \xi_A)] 
    = \mathrm{D} \pi(BG^{-1}, AG^{\top}) [(\xi_{BG^{-1}}, \xi_{AG^{\top}})],
\end{equation}
where $\pi : \mathcal{X} \rightarrow \mathcal{X}/\mathrm{GL}(r)$ is a natural projection.
By \cref{eq:horizon_space}, $(\xi_{BG^{-1}}, \xi_{AG^{\top}}) \in H_{(BG^{-1}, AG^{\top})}\mathcal{X}$ satisfies the horizontal space condition
\begin{align}\label{ap_eq:condition_for_horizon}
    G^{-\top} B^\top (\xi_{BG^{-1}}) G A^\top A G^\top = G^{-\top} B^\top B G^{-1} (\xi_{AG^{\top}}^\top) AG^{\top},
\end{align}
where $G^{-\top} \coloneqq (G^{\top})^{-1} = (G^{-1})^{\top}$.
Given a fixed $G \in \mathrm{GL}(r)$, define a linear map $\phi : (B,A) \mapsto (BG^{-1}, AG^{\top})$.
Then, $[(B,A)] = \pi(B,A) = \pi(\phi(B,A))$ and, by the chain rule,
\begin{align}
    \mathrm{D} \pi(B,A) [(\xi_B, \xi_A)] 
    &= \mathrm{D} \pi(\phi(B,A)) [\mathrm{D} \phi(B,A) (\xi_B, \xi_A)]\\
    &= \mathrm{D} \pi(BG^{-1}, AG^{\top}) [(\xi_B G^{-1}, \xi_A G^{\top})].\label{ap_eq:directional_derivative_2}
\end{align}
The pair $(\xi_B G^{-1}, \xi_A G^{\top})$ satisfies the horizontal space condition \cref{ap_eq:condition_for_horizon} because
\begin{align}
    &G^{-\top} B^\top (\xi_B G^{-1}) G A^\top A G^\top 
    = G^{-\top} \left( B^\top \xi_B A^\top A \right) G^\top, \\
    &G^{-\top} B^\top B G^{-1} (\xi_A G^{\top})^\top AG^{\top}
    = G^{-\top} \left( B^\top B  \xi_A^\top A \right) G^{\top}
    = G^{-\top} \left( B^\top \xi_B A^\top A \right) G^\top \quad (\because \cref{eq:horizon_space}).
\end{align}
Thus, $(\xi_B G^{-1}, \xi_A G^{\top})$ is also in the horizontal space $H_{(BG^{-1}, AG^{\top})}\mathcal{X}$. 
Because \cref{ap_eq:directional_derivative_1,ap_eq:directional_derivative_2}, and the uniqueness of the horizontal lifts, we obtain $(\xi_{BG^{-1}}, \xi_{AG^{\top}}) = (\xi_{B} G^{-1}, \xi_{A} G^{\top})$.
Considering $P_{BG^{-1}} = P_B$ and $P_{AG^\top} = P_A$, our metric \cref{eq:proposed_metric} yields
\begin{align}
    &g_{(BG^{-1}, AG^{\top})} \big( (\xi_{BG^{-1}},\xi_{AG^{\top}}), (\zeta_{BG^{-1}},\zeta_{AG^{\top}}) \big) \\
    & = \tr\left( (AG^\top)^\top AG^\top (\xi_{B} G^{-1})^\top \left( I + P_B \right) \zeta_{B} G^{-1} \right)\\
    &\qquad + \tr\left( (BG^{-1})^\top BG^{-1} (\xi_A G^\top)^\top \left( I + P_A \right) \zeta_A G^\top \right)\\
    & = \tr\left( A^\top A \xi_{B}^\top \left( I + P_B \right) \zeta_{B} \right)
    + \tr\left( B^\top B \xi_A^\top \left( I + P_A \right) \zeta_A \right)\\
    &= g_{(B,A)} \big( (\xi_B,\xi_A), (\zeta_B,\zeta_A) \big).
\end{align}
Therefore, our metric \cref{eq:proposed_metric} is invariant under the equivalence relation \cref{eq:invariant_GL(r)}.

\subsection{Proof of \cref{prop:projection}}\label{ap_sec:proof_projection}
The first-order change in (i.e., differential of) $W = W_0 + BA^\top$ with respect to $B$ and $A$, denoted as $\mathrm{d} W$, is given by  
\begin{equation}
    \mathrm{d} W = B(\mathrm{d} A)^\top + (\mathrm{d} B)A^\top,
\end{equation}
where $\mathrm{d} A \in \R^{n \times r}$ and $\mathrm{d} B \in \R^{m \times r}$.
Therefore, the first-order weight changes allowed by the LoRA parameterization \cref{eq:modeling_lora} form the subspace 
\begin{equation}\label{ap_eq:subsapce}
    \mathcal{S}_{(B,A)} = \{ B(\mathrm{d} A)^\top + (\mathrm{d} B)A^\top \mid \mathrm{d} A \in \R^{n \times r},\, \mathrm{d} B \in \R^{m \times r} \}.
\end{equation}
We here show the following lemma.
\begin{lemma}\label{ap_lemma:subspace_another_form}
    Let $\mathcal{S}_1 = \mathcal{S}_{(B,A)}$ and $\mathcal{S}_2 = \{ Y \in \R^{m \times n} \mid Q_B Y Q_A = O\}$.
    Then, $\mathcal{S}_1 = \mathcal{S}_2$.
\end{lemma}
\begin{proof}
    First, we show $\mathcal{S}_1 \subset \mathcal S_2$.
    Let $Y = B(\mathrm{d} A)^\top + (\mathrm{d} B)A^\top \in \mathcal{S}_1$.
    Then, it holds that $Q_B Y Q_A = O$ because $Q_B B = O$ and $A^\top Q_A = O$.
    This indicates that $Y$ can be written in the form defining $\mathcal{S}_2$, and thus $Y \in \mathcal{S}_2$.
    Therefore, $\mathcal{S}_1 \subset \mathcal{S}_2$.
    Next, we show $\mathcal{S}_2 \subset \mathcal S_1$.
    Let $Y \in \mathcal{S}_2$, and it satisfies 
    \begin{equation}
        Q_B Y Q_A = (I - P_B) Y (I - P_A) = Y - P_B Y - Y P_A + P_B Y P_A = O.
    \end{equation}
    Thus, $Y$ can be written as
    \begin{equation}
        Y = P_B Y + Y P_A - P_B Y P_A = P_B Y + (I - P_B) Y P_A.
    \end{equation}
    Consider $P_B = B (B^\top B)^{-1} B^\top$ and $P_A = A (A^\top A)^{-1} A^\top$, and define
    \begin{equation}
    \mathrm{d} A \coloneqq Y^\top B (B^\top B)^{-1},
    \qquad
    \mathrm{d} B \coloneqq (I - P_B) Y A (A^\top A)^{-1}.
    \end{equation}
    Then, we get 
    \begin{equation}
        P_B Y = B(\mathrm{d} A)^\top, \qquad (I - P_B) Y P_A = (\mathrm{d} B) A^\top,
    \end{equation}
    and hence
    \begin{equation}
        Y = B(\mathrm{d} A)^\top + (\mathrm{d} B) A^\top.
    \end{equation}
    Thus, $Y$ can be written in the form defining $\mathcal{S}_1$, and thus $Y \in \mathcal{S}_1$.
    Therefore, $\mathcal{S}_2 \subset \mathcal{S}_1$.
    From the above, we conclude $\mathcal{S}_1 = \mathcal{S}_2$.
\end{proof}
Therefore, using $\mathrm{d} W \in \R^{m \times n}$, \cref{ap_eq:subsapce} can be rewritten as 
\begin{equation}
    \mathcal{S}_{(B,A)} = \{ \mathrm{d} W \in \R^{m \times n} \mid Q_B (\mathrm{d} W) Q_A = O\}.
\end{equation}
Next, we consider the orthogonal projection of $X \in \R^{m \times n}$ onto the subspace $\mathcal{S}_{(B,A)}$ in terms of the Frobenius norm, defined as
\begin{equation}
    \operatorname{Proj}_{\mathcal{S}_{(B,A)}}(X) \coloneqq \argmin_{Y \in \mathcal{S}_{(B,A)}} \| X - Y \|_{F}^2.
\end{equation}
Considering $X = (P_B + Q_B) X (P_A + Q_A)$, we obtain 
\begin{align}
    \| X - Y \|_{F}^2 
    &= \| P_B X + X P_A - P_B X P_A + Q_B X Q_A - Y \|_{F}^2 \\
    &= \| P_B X + X P_A - P_B X P_A - Y \|_{F}^2 + \| Q_B X Q_A \|_{F}^2
\end{align}
because $P_B X + X P_A - P_B X P_A$ is orthogonal to $Q_B X Q_A$, and $Y$ is also orthogonal to $Q_B X Q_A$ by \cref{ap_lemma:subspace_another_form}, so that the corresponding traces are zeros.
Thus, we get 
\begin{align}
    \operatorname{Proj}_{\mathcal{S}_{(B,A)}}(X) 
    &= \argmin_{Y \in \mathcal{S}_{(B,A)}} \| X - Y \|_{F}^2 \\
    &= \argmin_{Y \in \mathcal{S}_{(B,A)}} \| P_B X + X P_A - P_B X P_A - Y \|_{F}^2 \\
    &= P_B X + X P_A - P_B X P_A.
\end{align}
Obviously, $P_B X + X P_A - P_B X P_A \in \mathcal{S}_{(B,A)}$ because $Q_B (P_B X + X P_A - P_B X P_A) Q_A = O$.

\subsection{Proof of \cref{prop:resulting_full_weight_matrix}}\label{ap_sec:proof_weight_comparison}
From \cref{eq:W_(t+1)_FFT,eq:W_(t+1)_LoRA,eq:W_(t+1)_LoRACP,eq:W_(t+1)_our}, we get
\begin{align}
    &W_{t+1}^{\operatorname{FFT}} - W_{t+1}^{\operatorname{LoRA}}
    = \alpha \big( P_{B_t}' (\nabla_{W_t} \mathcal{L}) + (\nabla_{W_t} \mathcal{L}) P_{A_t}' + P_{B_t} (\nabla_{W_t} \mathcal{L}) P_{A_t} \big) - \alpha Q_{B_t} (\nabla_{W_t} \mathcal{L}) Q_{A_t},\\
    &W_{t+1}^{\operatorname{FFT}} - W_{t+1}^{\operatorname{LoRA-CP}}
    = \alpha P_{B_t} (\nabla_{W_t} \mathcal{L}) P_{A_t} - \alpha Q_{B_t} (\nabla_{W_t} \mathcal{L}) Q_{A_t},\\
    &W_{t+1}^{\operatorname{FFT}} - W_{t+1}^{\operatorname{LoRA-OP}}
    = - \alpha Q_{B_t} (\nabla_{W_t} \mathcal{L}) Q_{A_t},
\end{align}
where $P_{B_t}' \coloneqq B_t B_t^\top - P_{B_t}$ and $P_{A_t}' \coloneqq A_t A_t^\top - P_{A_t}$ for simplicity here.
Because $Q_{B_t} P_{B_t} = Q_{B_t} P_{B_t}' = O$ and $Q_{A_t} P_{A_t} = Q_{A_t} P_{A_t}' = O$, the trace of any matrix expression involving these products is zero.
Then, we get
\begin{align}
    &\left\| W_{t+1}^{\operatorname{FFT}} - W_{t+1}^{\operatorname{LoRA}} \right\|_F^2 = \alpha^2 \left\| P_{B_t}' (\nabla_{W_t} \mathcal{L}) + (\nabla_{W_t} \mathcal{L}) P_{A_t}' + P_{B_t} (\nabla_{W_t} \mathcal{L}) P_{A_t} \right\|_F^2 + \alpha^2 \left\| Q_{B_t} (\nabla_{W_t} \mathcal{L}) Q_{A_t} \right\|_F^2, \\
    &\left\| W_{t+1}^{\operatorname{FFT}} - W_{t+1}^{\operatorname{LoRA-CP}} \right\|_F^2 = \alpha^2 \left\| P_{B_t} (\nabla_{W_t} \mathcal{L}) P_{A_t} \right\|_F^2 +
    \alpha^2 \left\| Q_{B_t} (\nabla_{W_t} \mathcal{L}) Q_{A_t} \right\|_F^2, \\
    &\left\| W_{t+1}^{\operatorname{FFT}} - W_{t+1}^{\operatorname{LoRA-OP}} \right\|_F^2 = \alpha^2 \left\| Q_{B_t} (\nabla_{W_t} \mathcal{L}) Q_{A_t} \right\|_F^2.
\end{align}
Therefore,
\begin{align}
    &\left\| W_{t+1}^{\operatorname{FFT}} - W_{t+1}^{\operatorname{LoRA}} \right\|_F^2 - \left\| W_{t+1}^{\operatorname{FFT}} - W_{t+1}^{\operatorname{LoRA-OP}} \right\|_F^2 \\
    &\quad = \alpha^2 \left\| P_{B_t}' (\nabla_{W_t} \mathcal{L}) + (\nabla_{W_t} \mathcal{L}) P_{A_t}' + P_{B_t} (\nabla_{W_t} \mathcal{L}) P_{A_t} \right\|_F^2 \geq 0.
\end{align}
Thus, 
\begin{align}
    \left\| W_{t+1}^{\operatorname{FFT}} - W_{t+1}^{\operatorname{LoRA-OP}} \right\|_F \leq \left\| W_{t+1}^{\operatorname{FFT}} - W_{t+1}^{\operatorname{LoRA}} \right\|_F,
\end{align}
and its inequality is strict whenever $P_{B_t}' (\nabla_{W_t} \mathcal{L}) + (\nabla_{W_t} \mathcal{L}) P_{A_t}' + P_{B_t} (\nabla_{W_t} \mathcal{L}) P_{A_t} \neq O$.
Also, 
\begin{align}
    \left\| W_{t+1}^{\operatorname{FFT}} - W_{t+1}^{\operatorname{LoRA-CP}} \right\|_F^2 - \left\| W_{t+1}^{\operatorname{FFT}} - W_{t+1}^{\operatorname{LoRA-OP}} \right\|_F^2
    = \alpha^2 \left\| P_{B_t} (\nabla_{W_t} \mathcal{L}) P_{A_t} \right\|_F^2 \geq 0.
\end{align}
Thus,
\begin{equation}
    \left\| W_{t+1}^{\operatorname{FFT}} - W_{t+1}^{\operatorname{LoRA-OP}} \right\|_F \leq \left\| W_{t+1}^{\operatorname{FFT}} - W_{t+1}^{\operatorname{LoRA-CP}} \right\|_F,
\end{equation}
and its inequality is strict whenever $P_{B_t} (\nabla_{W_t} \mathcal{L}) P_{A_t} \neq O$.

Note that 
\begin{align}
    &\left\| W_{t+1}^{\operatorname{FFT}} - W_{t+1}^{\operatorname{LoRA}} \right\|_F^2 - \left\| W_{t+1}^{\operatorname{FFT}} - W_{t+1}^{\operatorname{LoRA-CP}} \right\|_F^2 \\
    & = \alpha^2 \left\| P_{B_t}' (\nabla_{W_t} \mathcal{L}) + (\nabla_{W_t} \mathcal{L}) P_{A_t}' + P_{B_t} (\nabla_{W_t} \mathcal{L}) P_{A_t} \right\|_{F}^2 - \alpha^2 \left\| P_{B_t} (\nabla_{W_t} \mathcal{L}) P_{A_t} \right\|_{F}^2 \\
    & = \alpha^2 \left\| P_{B_t}' (\nabla_{W_t} \mathcal{L}) + (\nabla_{W_t} \mathcal{L}) P_{A_t}' \right\|_{F}^2 + 2 \alpha^2 \tr\left( \left( P_{B_t}' (\nabla_{W_t} \mathcal{L}) + (\nabla_{W_t} \mathcal{L}) P_{A_t}' \right)^\top P_{B_t} (\nabla_{W_t} \mathcal{L}) P_{A_t} \right).
\end{align}
This second trace term can be either positive or negative, and thus its sign (and hence the ordering) cannot be determined in general.

\section{Derivations}
In this section, we use the chain rule $\nabla_{B} \mathcal{L} = (\nabla_{W} \mathcal{L}) A$ and $\nabla_{A} \mathcal{L} = (\nabla_{W} \mathcal{L})^\top B$, and the relationship $W_t = W_0 + B_t A_t^\top$.

\subsection{Deviation of \cref{eq:grad_our}}\label{ap_subsec:deviation_our_riemannian_grad}
Under the LoRA case defined in \cref{eq:lora}, \cref{eq:def_rieman_grad} can be written as
\begin{equation}\label{apeq:def_rieman_grad}
    \mathrm{D}\mathcal{L}(W_0 + BA^\top)[(\xi_{B},\xi_A)] 
    = g_{(B,A)} \big( (\operatorname{grad}_{B} \mathcal{L}, \operatorname{grad}_{A} \mathcal{L}), (\xi_B,\xi_A) \big),
\end{equation}
for all $(\xi_B, \xi_A) \in T_{(B,A)} \mathcal{X}$.
The left side of \cref{apeq:def_rieman_grad} is written as
\begin{align}
    &\mathrm{D}\mathcal{L}(W_0 + BA^\top)[(\xi_{B},\xi_A)] \\
    &= \lim_{h \rightarrow 0}
    \frac{
        \mathcal{L}( W_0 + (B + h \xi_{B}) (A + h \xi_{A})^\top ) - \mathcal{L}( W_0 + BA^\top )
    }{
        h
    }\\
    &= \lim_{h \rightarrow 0}
    \frac{
        \mathcal{L}( W_0 + BA^\top ) 
        + h \inner{
                \nabla_{W}\mathcal{L}
            }{
                B \xi_{A}^\top + \xi_{B} A^\top
            }
        + O(h^2)
        -  \mathcal{L}( W_0 + BA^\top ) 
    }{
        h
    }\\
    &= \inner{
                \nabla_{W}\mathcal{L}
            }{
                B \xi_{A}^\top + \xi_{B} A^\top
        }\\
    &= \tr\left(
        ((\nabla_{W}\mathcal{L}) A)^{\top} \xi_{B} \right)
        + \tr\left( B^\top (\nabla_{W}\mathcal{L})\xi_{A} 
    \right)\\
    &= \tr\left((\nabla_{B}\mathcal{L})^{\top} \xi_{B}\right) + \tr\left( (\nabla_{A}\mathcal{L})^\top \xi_{A} \right) \\
    &= \inner*{(\nabla_{B}\mathcal{L}, \nabla_{A}\mathcal{L})}{(\xi_B, \xi_A)}. \label{apeq:leftside}
\end{align}
With our metric \cref{eq:proposed_metric}, the right side of \cref{apeq:def_rieman_grad} is written as
\begin{align}
    &g_{(B,A)} \big( (\operatorname{grad}_{B} \mathcal{L}, \operatorname{grad}_{A} \mathcal{L}), (\xi_B,\xi_A) \big)\\
    &= \tr\left( A^\top A (\operatorname{grad}_{B} \mathcal{L})^\top \left( I + P_B \right) \xi_B \right)
    + \tr\left( B^\top B (\operatorname{grad}_{A} \mathcal{L})^\top \left( I + P_A \right) \xi_A \right) \\
    &= \inner*{ \left( (I + P_B) (\operatorname{grad}_{B} \mathcal{L}) A^\top A,\, (I + P_A) (\operatorname{grad}_{A} \mathcal{L}) B^\top B \right) }{ (\xi_B, \xi_A) }.\label{apeq:rightside}
\end{align}
Therefore, \cref{apeq:def_rieman_grad} is rewritten as
\begin{gather}
    \inner*{(\nabla_{B}\mathcal{L} - (I + P_B) (\operatorname{grad}_{B} \mathcal{L}) A^\top A,\, \nabla_{A}\mathcal{L} - (I + P_A) (\operatorname{grad}_{A} \mathcal{L}) B^\top B)}{(\xi_B, \xi_A)} = 0,
\end{gather}
for all $(\xi_B, \xi_A) \in T_{(B,A)} \mathcal{X}$.
Because $T_{(B,A)} \mathcal{X} = T_{(B,A)} (\R^{m \times r} \times \R^{n \times r}) = \R^{m \times r} \times \R^{n \times r}$, it is necessary that
\begin{equation}
    \nabla_{B}\mathcal{L} - (I + P_B) (\operatorname{grad}_{B} \mathcal{L}) A^\top A = O, \quad \nabla_{A}\mathcal{L} - (I + P_A) (\operatorname{grad}_{A} \mathcal{L}) B^\top B = O.
\end{equation}
Considering $(I + P_B)^{-1} = (I -\tfrac{1}{2} P_B)$ and $(I + P_A)^{-1} = (I -\tfrac{1}{2} P_A)$, we obtain
\begin{equation}
    \operatorname{grad}_{B} \mathcal{L} = (I - \tfrac{1}{2}P_B) \nabla_{B}\mathcal{L} (A^\top A)^{-1},\quad
    \operatorname{grad}_{A} \mathcal{L} =  (I - \tfrac{1}{2} P_A) \nabla_{A}\mathcal{L} (B^\top B)^{-1}.
\end{equation}
This Riemannian gradient $(\operatorname{grad}_{B} \mathcal{L}, \operatorname{grad}_{A} \mathcal{L})$ lies in the horizontal space $H_{(B,A)}\mathcal{X}$ defined by \cref{eq:horizon_space} because it satisfies the horizontal space condition $B^\top (\operatorname{grad}_{B} \mathcal{L}) A^\top A = B^\top B (\operatorname{grad}_{A} \mathcal{L})^\top A$:
\begin{align}
    B^\top (\operatorname{grad}_{B} \mathcal{L}) A^\top A 
    &= \frac{1}{2} B^\top (\nabla_{B}\mathcal{L})
    = \frac{1}{2} B^\top (\nabla_{W}\mathcal{L}) A,\\
    B^\top B (\operatorname{grad}_{A} \mathcal{L})^\top A 
    &= \frac{1}{2} (\nabla_{A}\mathcal{L})^\top A
    = \frac{1}{2} B^\top (\nabla_{W}\mathcal{L}) A.
\end{align}

\subsection{Derivation of \cref{eq:W_(t+1)_LoRACP}}\label{ap_sec:derivation_of_W_LoRACP}
\begin{align}
    W_{t+1}^{\operatorname{LoRA-CP}}
    &= W_0 + \left(B_{t} - \alpha (\nabla_{B_t} \mathcal{L}) (A_t^\top A_t)^{-1} \right) \left(A_{t} - \alpha (\nabla_{A_t} \mathcal{L}) (B_t^\top B_t)^{-1} \right)^\top\\
    &= W_0 + B_t A_t^\top -\alpha B_t (B_t^\top B_t)^{-1} (\nabla_{A_t} \mathcal{L})^\top - \alpha  (\nabla_{B_t} \mathcal{L}) (A_t^\top A_t)^{-1} A_t^\top + O(\alpha^2)\\
    &= W_t -\alpha B_t (B_t^\top B_t)^{-1} B_t^\top (\nabla_{W_t} \mathcal{L}) - \alpha  (\nabla_{W_t} \mathcal{L}) A_t (A_t^\top A_t)^{-1} A_t^\top + O(\alpha^2)\\
    &= W_t - \alpha P_{B_t} (\nabla_{W_t} \mathcal{L}) - \alpha (\nabla_{W_t} \mathcal{L}) P_{A_t} + O(\alpha^2)\\
    &\approx W_t - \alpha P_{B_t} (\nabla_{W_t} \mathcal{L}) - \alpha (\nabla_{W_t} \mathcal{L}) P_{A_t}.
\end{align}

\subsection{Derivation of \cref{eq:W_(t+1)_our}}\label{ap_sec:derivation_of_W_our}
\begin{align}
    W_{t+1}^{\operatorname{LoRA-OP}}
    &= W_0 + \left(B_{t} - \alpha ( I - \tfrac{1}{2} P_{B_t} ) (\nabla_{B_t} \mathcal{L}) (A_t^\top A_t)^{-1} \right) \left(A_{t} - \alpha ( I - \tfrac{1}{2} P_{A_t} ) (\nabla_{A_t} \mathcal{L}) (B_t^\top B_t)^{-1} \right)^\top\\
    &= W_0 + B_t A_t^\top -\alpha B_t (B_t^\top B_t)^{-1} (\nabla_{A_t} \mathcal{L})^\top - \alpha  (\nabla_{B_t} \mathcal{L}) (A_t^\top A_t)^{-1} A_t^\top \\
    & \qquad 
    + \frac{\alpha}{2} P_{B_t}  (\nabla_{B_t} \mathcal{L}) (A_t^\top A_t)^{-1} A_t^\top 
    + \frac{\alpha}{2} B_t (B_t^\top B_t)^{-1} (\nabla_{A_t} \mathcal{L})^\top P_{A_t} 
    + O(\alpha^2)
    \\
    &= W_t -\alpha B_t (B_t^\top B_t)^{-1} B_t^\top (\nabla_{W_t} \mathcal{L}) - \alpha  (\nabla_{W_t} \mathcal{L}) A_t (A_t^\top A_t)^{-1} A_t^\top \\
    & \qquad 
    + \frac{\alpha}{2} P_{B_t} (\nabla_{W_t} \mathcal{L}) A_t (A_t^\top A_t)^{-1} A_t^\top 
    + \frac{\alpha}{2} B_t (B_t^\top B_t)^{-1} B_t^\top (\nabla_{W_t} \mathcal{L}) P_{A_t} 
    + O(\alpha^2)\\
    &= W_t - \alpha P_{B_t} (\nabla_{W_t} \mathcal{L}) - \alpha (\nabla_{W_t} \mathcal{L}) P_{A_t} + \alpha P_{B_t} (\nabla_{W_t} \mathcal{L}) P_{A_t} + O(\alpha^2)\\
    &\approx W_t - \alpha P_{B_t} (\nabla_{W_t} \mathcal{L}) - \alpha (\nabla_{W_t} \mathcal{L}) P_{A_t} + \alpha P_{B_t} (\nabla_{W_t} \mathcal{L}) P_{A_t}.
\end{align}

\section{Datasets}\label{ap_sec:datasets}
\subsection{Natural Language Generation}
In \cref{subsec:NLG}, we used the E2E dataset for the natural language generation challenge proposed in \cite{novikova2017e2e}.
This dataset provides information about restaurants and consists of 51,426 dialogue-act-based meaning representations, each with an average of 8.1 reference sentences.
The dataset is split into train, validation, and test sets, with 42,061, 4,672, and 4,693 instances, respectively.
This splitting keeps a similar distribution of meaning representation and reference text lengths on each set.

\subsection{Image Classification}
In \cref{subsec:image_classification}, we used the seven image classification datasets, such as Cars \cite{Cars}, DTD \cite{DTD}, EuroSAT \cite{EuroSAT}, GTSRB \cite{GTSRB}, RESISC45 \cite{RESISC45}, SUN397 \cite{SUN397}, and SVHN \cite{SVHN}.
These datasets are split into the train and test sets, except for EuroSAT and RESISC45.
For such datasets, we first split the train set into the train and validation sets in a 9-1 ratio, and then prepared the train, validation, and test sets.
EuroSAT consists of 27,000 samples, and we split it into the train, validation, and test sets in a 6-1-3 ratio.
RESISC45 originally provides the train, validation, and test sets.
Thus, we used these sets.
Note that we used $\operatorname{train\_test\_split}$ function with the $\operatorname{stratify}$ parameter in the scikit-learn library \cite{scikit-learn} for the splitting because it can ensure the same proportion of classes (or labels) in the split sets.
We summarize the number of samples in each dataset as follows.
\begin{table*}[h]
    \caption{The number of samples in the train, validation, and test sets of each image classification dataset used in \cref{subsec:image_classification}}
    \setlength{\tabcolsep}{6pt} 
    \centering
    \begin{tabular}[t]{lccccccccc}
        \toprule
        \# of Samples
        & Cars & DTD & EuroSAT & GTRSB & RESISC45 & SUN397 & SVHN\\
        \midrule
        Train
        & 7,329 & 1,692 & 16,200 & 23,976 & 18,900 & 17,865 & 65,931\\
        Validation
        & 815 & 188 & 2,700 & 2,664 & 6,300 & 1,985 & 7,326\\
        Test
        & 8,041 & 1,880 & 8,100 & 12,630 & 6,300 & 19,850 & 26,032\\
        \bottomrule
    \end{tabular}
\end{table*}

\subsection{Commonsense Reasoning}
In \cref{subsec:CR}, we used the eight commonsense reasoning datasets, such as BoolQ \cite{clark-etal-2019-boolq}, PIQA \cite{piqa}, SIQA \cite{sap-etal-2019-siqa}, HellaSwag (HellaS in the main paper) \cite{zellers-etal-2019-hellaswag}, WinoGrande (WinoG in the main paper) \cite{winogrande}, ARC-e \cite{arc-e_and_arc_c}, ARC-c \cite{arc-e_and_arc_c}, and OBQA \cite{mihaylov-etal-2018-obqa}.
Following \cite{li2025stiefel,park2025stiefel,chenghao2025greatagain}, we also used the combined dataset of all train sets, i.e., Commonsense170K, which consists of 170,420 samples.
However, for fair comparison, we split Commonsense170K randomly into the train and validation sets in 160,420 and 10,000 samples, respectively, and then tuned the learning rate of each method using the validation set.
After the train and validation, we evaluated each fine-tuned model on each individual test set.
We summarize the number of samples in the each individual test set as follows.
\begin{table*}[h]
    \caption{The number of samples in the test set of each commonsense reasoning dataset used in \cref{subsec:CR}}
    \setlength{\tabcolsep}{5pt} 
    \centering
    \begin{tabular}[t]{lcccccccccc}
        \toprule
        \# of Samples
        & BoolQ & PIQA & SIQA & HellaS & WinoG & ARC-e & ARC-c & OBQA \\
        \midrule
        Test
        & 1172 & 2376 & 3270 & 10042 & 500 & 1838 & 1954 & 1267 \\
        \bottomrule
    \end{tabular}
\end{table*}

\section{Details of Hyperparameters}\label{ap_sec:hyperparameter}
We provide the details of hyperparameters used in our experiments to ensure reproducibility. 
For each task, all compared methods share a common set of hyperparameters, such as the rank, batch size, weight decay, and training schedule, which are outlined in \cref{sec:experiment}.
When decoding a sentence in \cref{subsec:NLG}, we used the beam search with the beam size to $10$, the length penalty to $0.8$, and the no-repeat-Ngram-size to $4$, following \cite{zhang2024riemann}.
For answer generation in \cref{subsec:CR}, we used greedy decoding with a maximum of 30 generated tokens.
As mentioned in \cref{sec:experiment}, we individually tune the learning rate for each method on each validation set to ensure fair comparison.
\Cref{ap_tab:NLG} lists the tuned learning rates for the natural language generation task in \cref{subsec:NLG}.
In this task, we tuned the learning rate over $\{0.02, 0.04, \dots, 0.30\}$ for SGD and $\{0.0002, 0.0004, \dots, 0.0030\}$ for AdamW methods on the validation set.
\Cref{ap_tab:learning_rate_IC,ap_tab:CR} list the tuned learning rates for the image classification task in \cref{subsec:image_classification} and the commonsense reasoning task in \cref{subsec:CR}, respectively.
In these tasks, we tuned the learning rate over $\{0.0001, 0.0002, \dots, 0.001\}$ for AdamW methods on the validation set.
Note that, in the natural language generation task of \cref{subsec:NLG}, we selected the learning rate that achieves the highest average rank (and the lowest rank variance in case of ties in the highest average rank) across the five evaluation metrics (BLEU \cite{papineni-etal-2002-bleu}, NIST \cite{NIST}, METEOR \cite{banerjee-lavie-2005-meteor}, ROUGE-L \cite{lin-2004-rouge}, and CIDEr \cite{cider}) on the validation set.
In the image classification task of \cref{subsec:image_classification}, we selected the learning rate that achieves the top-1 validation accuracy on each dataset.
In the commonsense reasoning task of \cref{subsec:CR}, we selected the learning rate that achieves the top-1 validation accuracy on the validation set which is obtained by splitting Commonsense170K.

\begin{table*}[h]
    \caption{
        Learning rates for the natural language generation task in \cref{subsec:NLG}
    }\label{ap_tab:NLG}
    \centering
    \begin{tabular}[t]{lcccccccccc}
        \toprule
        Method & Learning rate \\
        \midrule
        SGD & 0.08 \\
        w/ \cref{eq:preconditioning_quotient} & 0.06 \\
        w/ \cref{eq:preconditioning_our} (Ours) & 0.02 \\        
        \midrule
        \midrule
        AdamW & 0.0028 \\
        w/ \cref{eq:preconditioning_quotient} & 0.0028 \\
        w/ \cref{eq:preconditioning_our} (Ours) & 0.0030 \\
        \bottomrule
    \end{tabular}
\end{table*}

\begin{table*}[h]
    \caption{Learning rates for the image classification task in \cref{subsec:image_classification}}
    \label{ap_tab:learning_rate_IC}
    \setlength{\tabcolsep}{6pt} 
    \centering
    \begin{tabular}[t]{lccccccccc}
        \toprule
        \# of Samples
        & Cars & DTD & EuroSAT & GTRSB & RESISC45 & SUN397 & SVHN\\
        \midrule
        AdamW
        & 0.0008 & 0.0006 & 0.001 & 0.0007 & 0.001 & 0.0005 & 0.001\\
        w/ \cref{eq:preconditioning_quotient}
        & 0.0009 & 0.0008 & 0.0009 & 0.0009 & 0.001 & 0.0007 & 0.001\\
        w/ \cref{eq:preconditioning_our} (Ours)
        & 0.0009 & 0.001 & 0.001 & 0.001 & 0.001 & 0.0007 & 0.001\\
        \bottomrule
    \end{tabular}
\end{table*}

\begin{table*}[h]
    \caption{
        Learning rates for the commonsense reasoning task in \cref{subsec:CR}
    }\label{ap_tab:CR}
    \centering
    \begin{tabular}[t]{lcccccccccc}
        \toprule
        Method & Learning rate \\
        \midrule
        AdamW & 0.0001 \\
        w/ \cref{eq:preconditioning_quotient} & 0.0001 \\
        w/ \cref{eq:preconditioning_our} (Ours) & 0.0001 \\
        \bottomrule
    \end{tabular}
\end{table*}

\section{Relationship with LoRA-Pro~\cite{wang2025lorapro}}\label{ap_sec:LoRAPro}
LoRA-Pro~\cite{wang2025lorapro} aims to align the gradients of LoRA with those of full fine-tuning, and then proposes to adjust the Euclidean gradients $\nabla_{B} \mathcal{L}$ and $\nabla_{A} \mathcal{L}$ at each step.
This aim differs from but is similar to our metric design principle, i.e., closing the gap to full fine-tuning at the weight level.
In this section, we discuss the relationship between the adjusted gradients of LoRA-Pro and our preconditioned gradients, i.e., our Riemannian gradients \cref{eq:grad_our}.

Let $\nabla_{B}^{\operatorname{Pro}} \mathcal{L}$ and $\nabla_{A}^{\operatorname{Pro}} \mathcal{L}$ be the adjusted gradients of $B$ and $A$ of LoRA-Pro.
Then, from \cite[(33, 34) in Appendix B.2]{wang2025lorapro}, these gradients can be written under our LoRA parametrization \cref{eq:modeling_lora} as
\begin{equation}\label{ap_eq:grad_lorapro}
\begin{aligned}
    &\nabla_{B}^{\operatorname{Pro}} \mathcal{L} = ( I - P_B ) (\nabla_{B} \mathcal{L}) (A^\top A)^{-1} + BX, \\
    &\nabla_{A}^{\operatorname{Pro}} \mathcal{L} = (\nabla_{A} \mathcal{L}) (B^\top B)^{-1} - AX^{\top},
\end{aligned}
\end{equation}
where $X \in \R^{r \times r}$ represents an arbitrary matrix.
We here set the scaling factor $s=1$ for simplicity.
To obtain a unique $X$ at each gradient step, \citet{wang2025lorapro} proposed to solve the following Sylvester equation at each gradient step:
\begin{equation}
    B^\top B X + X A^\top A = -(B^\top B)^{-1} (\nabla_{A} \mathcal{L})^\top A.
\end{equation}
For details, see \cite[Theorem 2.3.]{wang2025lorapro}.
This equation is derived from the motivation that $\nabla_{B}^{\operatorname{Pro}} \mathcal{L}$ and $\nabla_{A}^{\operatorname{Pro}} \mathcal{L}$ remain close to $\nabla_{B} \mathcal{L}$ and $\nabla_{A} \mathcal{L}$, respectively.
To obtain a unique $X$, we instead consider that $\left(\nabla_{B}^{\operatorname{Pro}} \mathcal{L}, \nabla_{A}^{\operatorname{Pro}} \mathcal{L} \right)$ is on the horizontal space defined by \cref{eq:horizon_space}.
For this, it is necessary that
\begin{equation}
    B^\top (\nabla_{B}^{\operatorname{Pro}} \mathcal{L}) A^\top A = B^\top B (\nabla_{A}^{\operatorname{Pro}} \mathcal{L})^\top A.
\end{equation}
By solving this, we get
\begin{equation}\label{ap_eq:uniqueX}
    X = \frac{1}{2} (B^\top B)^{-1} (\nabla_{A} \mathcal{L})^\top A (A^\top A)^{-1}
    = \frac{1}{2} (B^\top B)^{-1} B^\top (\nabla_{B} \mathcal{L}) (A^\top A)^{-1}.
\end{equation}
Plugging \cref{ap_eq:uniqueX} into \cref{ap_eq:grad_lorapro} yields
\begin{equation}
    \nabla_{B}^{\operatorname{Pro}} \mathcal{L} = ( I - \tfrac{1}{2}P_B ) (\nabla_{B} \mathcal{L}) (A^\top A)^{-1}, \quad
    \nabla_{A}^{\operatorname{Pro}} \mathcal{L} = ( I - \tfrac{1}{2}P_A ) (\nabla_{A} \mathcal{L}) (B^\top B)^{-1}.
\end{equation}
Obviously, this is identical to our Riemannian gradient \cref{eq:grad_our}.
This result indicates that our preconditioning is the variant of the gradient adjustment in LoRA-Pro so that exploits the geometry of the quotient manifold $\mathcal{X}/\mathrm{GL}(r)$ inherent to the LoRA parametrization \cref{eq:modeling_lora}.
Notably, this result does not imply that LoRA-Pro can be directly interpreted as Riemannian preconditioning.
This is because $(I - P_B)$ in \cref{ap_eq:grad_lorapro} is not invertible.
Since Riemannian preconditioning is given by an invertible linear transformation induced by a Riemannian metric, no such metric on $\mathcal{X}$ can induce \cref{ap_eq:grad_lorapro} on $T_{(B,A)}\mathcal{X}$.

\section{Empirical Validation of \cref{prop:resulting_full_weight_matrix}}\label{ap_sec:experiment_for_proposition}
To validate \cref{prop:resulting_full_weight_matrix} empirically, we compared the updated weight matrices of full fine-tuning and LoRA with different preconditioning, i.e., $\left\| W_{t+1}^{\operatorname{FFT}} - W_{t+1}^{\operatorname{LoRA}} \right\|_F$, $\left\| W_{t+1}^{\operatorname{FFT}} - W_{t+1}^{\operatorname{LoRA-CP}} \right\|_F$, and $\left\| W_{t+1}^{\operatorname{FFT}} - W_{t+1}^{\operatorname{LoRA-OP}} \right\|_F$.
\Cref{prop:resulting_full_weight_matrix} depends on that all methods share the common $t$-th weight matrix $W_{t}$ and trainable factors $B_t$ and $A_t$.
To implement this setting, we conducted a single-step fine-tuning task for the GPT2-medium model on the E2E natural language generation challenge dataset, like \cref{subsec:NLG}.
In this experiment, we added uniformly random perturbations to $W_{0}$, $B_{0}$, and $A_{0}$ in the range $[-0.01, 0.01)$, and shared them across all methods.
Each method then performed a single optimization step using the same randomly sampled batch.
As a result, we obtained $W_{1}^{\operatorname{FFT}}$, $W_{1}^{\operatorname{LoRA}}$, $W_{1}^{\operatorname{LoRA-CP}}$, and $W_{1}^{\operatorname{LoRA-OP}}$ from the common random initialization and batch.
We used the SGD methods because \cref{prop:resulting_full_weight_matrix} is based on the gradient descent update.
We here tested three learning rates $\{0.02, 0.06, 0.08\}$ that were tuned and selected on the validation set as described in \cref{ap_tab:NLG}. 
To follow the assumption in \cref{sec:theoretical_insights}, we fixed the scaling factor of $BA^\top$ to $1$, the dropout to $0.0$, and the weight decay to $0.0$.
All other hyperparameters were set as described in \cref{subsec:NLG}.
We repeated this experiment 500 times and computed the average Frobenius norm over all runs and layers of the model.

\Cref{ap_tab:weight} shows the results.
We observed that $\left\| W_{1}^{\operatorname{FFT}} - W_{1}^{\operatorname{LoRA-OP}} \right\|_F \leq \left\| W_{1}^{\operatorname{FFT}} - W_{1}^{\operatorname{LoRA}} \right\|_F$ and $\left\| W_{1}^{\operatorname{FFT}} - W_{1}^{\operatorname{LoRA-OP}} \right\|_F \leq \left\| W_{1}^{\operatorname{FFT}} - W_{1}^{\operatorname{LoRA-CP}} \right\|_F$ at each learning rate setting.
Although a more detailed analysis across various experimental settings would be desirable, these results provide some empirical evidence for \cref{prop:resulting_full_weight_matrix}.


\begin{table*}[h]
    \caption{
        Comparisons of the Frobenius norm.
    }\label{ap_tab:weight}
    \centering
    \renewcommand{\arraystretch}{1.5}
    \begin{tabular}[t]{cccc}
        \toprule
        Learning rate & 0.02 & 0.06 & 0.08 \\
        \midrule
        $\left\| W_{1}^{\operatorname{FFT}} - W_{1}^{\operatorname{LoRA}} \right\|_{F}$ 
        & 10647.35 $\times 10^{-5}$ & 16542.16 $\times 10^{-5}$ & 20334.76 $\times 10^{-5}$ \\
        $\left\| W_{1}^{\operatorname{FFT}} - W_{1}^{\operatorname{LoRA-CP}} \right\|_{F}$ 
        & 10646.05 $\times 10^{-5}$ & 16534.73 $\times 10^{-5}$ & 20324.08 $\times 10^{-5}$ \\
        $\left\| W_{1}^{\operatorname{FFT}} - W_{1}^{\operatorname{LoRA-OP}} \right\|_{F}$ 
        & 10646.04 $\times 10^{-5}$ & 16534.69 $\times 10^{-5}$ & 20324.02 $\times 10^{-5}$\\ 
        \bottomrule
    \end{tabular}
\end{table*}


\end{document}